\ifdefined\pdfminorversion\pdfminorversion=7\fi
\documentclass{article}

\usepackage[main, final]{neurips_2026}

\usepackage[utf8]{inputenc} 
\usepackage[T1]{fontenc}    
\usepackage{hyperref}       
\usepackage{url}            
\usepackage{booktabs}       
\usepackage{amsmath,amssymb,amsfonts,amsthm,mathtools,tabularx}
\usepackage{nicefrac}       
\usepackage{microtype}      
\usepackage{xcolor}         
\usepackage{graphicx}
\usepackage{placeins}
\usepackage{array}
\usepackage{needspace}
\usepackage{tikz}
\usetikzlibrary{shapes.geometric,arrows.meta,calc,matrix}
\usepackage{pgfplots}
\pgfplotsset{compat=1.18}

\hypersetup{
  colorlinks=true,
  linkcolor=blue!50!black,
  citecolor=blue!50!black,
  urlcolor=blue!50!black,
  pdftitle={Geometry Matters in Packed 3D Attention},
  pdfauthor={Amil Khan, Anantajit Subrahmanya, B.S. Manjunath}
}

\theoremstyle{plain}
\newtheorem{theorem}{Theorem}
\newtheorem{lemma}[theorem]{Lemma}
\newtheorem{proposition}[theorem]{Proposition}
\newtheorem{corollary}[theorem]{Corollary}

\definecolor{cAxial}{HTML}{2A78D6}
\definecolor{cPacked}{HTML}{EB6834}
\definecolor{cSep}{HTML}{1BAF7A}
\definecolor{cInk}{HTML}{1A1A1A}
\definecolor{cMuted}{HTML}{6B6A66}
\definecolor{cGrid}{HTML}{E4E3DF}
\definecolor{shA}{HTML}{F2F1ED}
\definecolor{shB}{HTML}{E0DED6}
\definecolor{shC}{HTML}{CBC9BF}
\definecolor{shD}{HTML}{B5B3A8}

\newcommand{\method}{\textsc{Pack3D}}
\newcommand{\attn}{\operatorname{Attn}}
\newcommand{\softmax}{\operatorname{softmax}}
\newcommand{\LN}{\operatorname{LN}}


\newcolumntype{L}[1]{>{\raggedright\arraybackslash}p{#1}}
\newcolumntype{Y}{>{\raggedright\arraybackslash}X}

\title{Geometry Matters in Packed 3D Attention}

\author{%
  Amil Khan \qquad Anantajit Subrahmanya \qquad B.S. Manjunath\\
  Department of Electrical \& Computer Engineering\\
  University of California, Santa Barbara\\
  \texttt{\{amil, anantajit, manj\}@ucsb.edu}
}

\begin{document}

\maketitle

\begin{abstract}
Packing a 3D field into a token sequence gives every token two indices: a coordinate in the
sample and a position in memory. Raster, Morton, and block-contiguous layouts change the
position and leave the coordinate alone. A positional mechanism indexed by sequence position
can therefore learn the memory layout rather than the object, and accuracy on the training
layout cannot reveal that it has. We introduce \method{}, a local--global attention operator
that separates memory layout from attention geometry. It attends exactly within a
coordinate-defined neighborhood, sees the whole field through pooled summaries, and rotates
queries and keys with axial RoPE at token-grid coordinates. Because keys are rotated before
pooling, each query-to-summary logit is a weighted sum of query--key terms evaluated at the
summarized tokens' displacements from the query. The operator is equivariant to any repacking
that carries coordinates, neighborhoods, summary memberships, and within-block slots along
with the tokens. {{{When evaluated on}}} 
Allen WTC-11 microscopy at 262k tokens, \method{} and a matched packed-sequence RoPE control
reach comparable Dice on the training layout. Repacking the same validation volumes then costs
the packed control 0.048 macro Dice on average, with a loss in all six seeds, while \method{}
changes by 0.0001. On MSD Heart and Spleen, sequence indexing again changes more predictions
under repacking, at little cost in Dice. \method{} also exceeds a parameter-matched 3D U-Net
by 0.009 Dice over six seeds, and a full \method{} layer runs within 7--31\% of window-attention latency
while adding a global pathway. Geometry, not storage order, should index attention over packed
3D fields. Code is available at \url{https://github.com/amilworks/Pack3D}.
\end{abstract}

\section{Introduction}
\label{sec:intro}
High-token 3D fields arise across volumetric domains: live-cell microscopy
\citep{vianaIntegratedIntracellularOrganization2023}, medical imaging
\citep{antonelliMedicalSegmentationDecathlon2022,fangLargeModelsMedical2025}, learned
compression of microscopy volumes \citep{zhouDeeplearningbasedImageCompression2024}, and
spectral--spatial sensing cubes such as hyperspectral methane detection
\citep{kumarMethaneMapperSpectralAbsorption2023}. In each, a volume is patchified and
linearized into a token sequence, and linearization discards something physical: geometry no
longer coincides with token order in memory. Tokens may be packed in raster, Morton, or
block-contiguous order for systems efficiency, yet these layouts are not coordinate systems.
When a positional mechanism reads the layout as if it were geometry, attention can learn the
memory layout rather than the object. We call this the packed-geometric attention problem.

Attention is attractive in these fields because local decisions often depend on nonlocal
context. Convolutional and U-shaped networks remain strong segmenters because they hard-code a
local spatial bias; attention is a complementary primitive that mixes local and global
information while retaining spatial structure \citep{raghuVisionTransformersSee2021}. The
relevant question is therefore not whether attention should replace convolution, but what
geometry attention should use once a 3D field has been packed into a sequence.

Dense and window attention fail in opposite ways. Dense attention keeps global access but
scales quadratically with token count; window attention scales but sees only nearby tokens.
Local--global attention, exact nearby and summarized far away, is the natural compromise, and
it is where the problem bites: the summaries inherit whatever coordinate the positional
mechanism uses. Reducing the attention graph and fixing its geometry are separate
requirements; an operator can meet the first and fail the second.

Three-dimensional microscopy is where the problem is unavoidable, and it is our principal
domain. A single live-cell crop holds hundreds of thousands of tokens, and the signal spans
puncta, membranes, cytoplasmic networks, nuclear geometry, and cell-scale organization. The
Allen WTC--11 hiPSC Single-Cell Image Dataset v1 contains more than 200{,}000 live cells in 3D
across 25 tagged structures, and its analysis framework is built on cell shape, intracellular
coordinates, average structure locations, and pairwise structure interactions
\citep{vianaIntegratedIntracellularOrganization2023}. Our main task, Allen-15, covers 15 of
these structures (\S\ref{subsec:results-contract}). When the scientific objects are themselves
spatial, a model that confuses storage order with location is not merely inefficient; it is
answering the wrong question.

\begin{figure}[t]
    \centering
    \includegraphics[width=\linewidth]{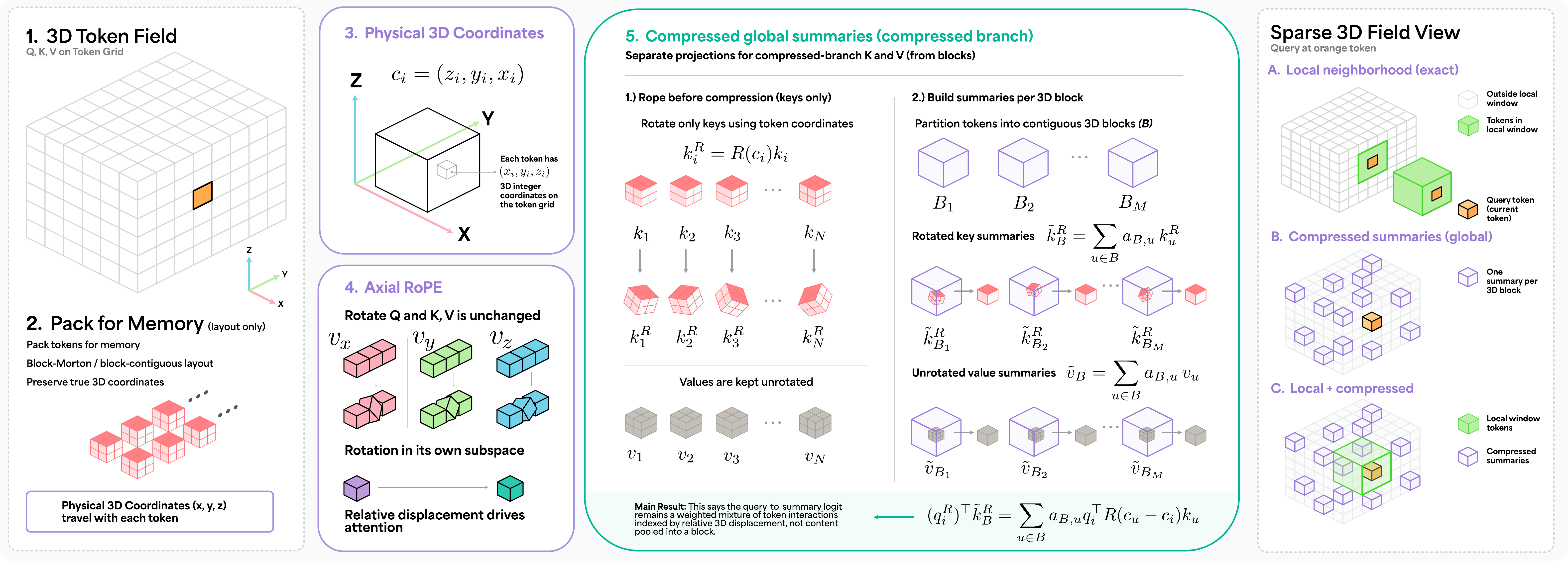}
    \caption{\textbf{\method{} separates memory layout from attention geometry.}
    \emph{Panels 1--2:} a 3D token field may be packed in any memory order for efficient
    execution, but each token's grid coordinate \(c_i=(z_i,y_i,x_i)\) travels with it.
    \emph{Panels 3--4:} axial RoPE rotates queries and keys by their coordinates, one
    subspace per axis, and leaves values unrotated, so query--key compatibility depends on
    the relative displacement \(c_j-c_i\), not on sequence position.
    \emph{Panel 5:} the compressed branch rotates keys \emph{before} pooling them into block
    summaries, so each query-to-summary logit is a weighted sum of query--key terms evaluated at
    the displacements of the tokens that formed the summary (Proposition~\ref{prop:relative}).
    \emph{Right:} each query (orange) attends exactly to its coordinate neighborhood (A) and
    to summaries across the field (B), combining exact local and compressed global context
    (C). The schematic draws one summary per block; the reported models add a second
    pooling stage and use the layouts described in \S\ref{subsec:method-repacking}.}
    \label{fig:pack3d}
\end{figure}

The difficulty is that this failure is invisible where models are usually evaluated. On any
one fixed layout, sequence position is a consistent proxy for geometry: under raster order,
for instance, a sequence offset is a fixed linear function of the 3D offset. A
sequence-indexed model can therefore fit the field as well as a coordinate-indexed one. The
two differ only when the packing changes, which systems code does routinely when it retiles a
kernel, switches between raster and block-Morton order, or reshards a volume. The direct test
holds a trained checkpoint and the physical field fixed, changes only the memory order, and
compares predictions after restoring a common spatial order (Figure~\ref{fig:repacking}).

We introduce \method{}, a geometry-aware local--global attention operator for packed
high-token 3D fields (Figure~\ref{fig:pack3d}). \method{} keeps attention exact within a
coordinate neighborhood, sees the whole field through pooled summaries, and indexes every
query--key interaction with axial RoPE on token-grid coordinates. What makes this more than a renaming of position is that compression preserves query-relative
geometry. For a summary \(g\) pooled with nonnegative weights \(A_{gu}\) that sum to one,
\[
    (q_i^R)^\top \widetilde{k}_g^R = \sum_{u} A_{gu}\, q_i^\top R(c_u-c_i)\,k_u .
\]
A query does not attend to an anonymous pooled block; it attends to block content through the
relative 3D displacement of the tokens that formed the summary. Rotating before pooling is what
keeps the summary spatial. The systems contract is part of the method. A sparse pattern that materializes its neighborhoods is sparse only algebraically: it still moves a gathered
\(N\times w\times d_h\) tensor through memory. Our Triton kernels therefore stream neighborhood tiles instead of gathering them,
following the lesson that useful sparsity must align with memory movement and the backward
pass, not only with asymptotic cost \citep{yuanNativeSparseAttention2025b}.

The claim is narrow, and it is reusable. The Allen task predicts Allen-generated masks from
target fluorescence, brightfield, and DNA; it is neither label-free prediction nor manual-label
segmentation. We do not claim segmentation state of the art, a fixed-layout accuracy advantage
for one positional encoding, or a general speed law for sparse attention. We claim that a
positional mechanism keyed to memory order can bind a trained model to its storage layout, and
that indexing attention by coordinates removes this dependence at no measured cost in
fixed-layout accuracy. Our contributions are:
\begin{enumerate}
    \item \textbf{A design principle and an operator.} We formulate \method{} and make the
    separation of memory layout from attention geometry precise: a repacking leaves the
    operator unchanged when it carries coordinates, neighborhood edges, summary memberships,
    and slot identities with the tokens (\S\ref{subsec:method-repacking}).

    \item \textbf{The geometry of compressed attention.} Rotating keys before pooling makes
    every query-to-summary logit an exact weighted sum of query--key terms evaluated at member
    displacements, through both pooling stages. A summary-center variant is a cheaper
    coordinate-aware approximation; natural-data Dice does not separate the two, but a controlled
    probe does
    (\S\ref{subsec:results-branch}).

    \item \textbf{A hardware-aligned realization.} Triton local attention avoids materializing
    neighborhoods. Per layer, \method{} is 6.4--99.6\(\times\) faster than dense attention from
    65k to 1M tokens and within 7--31\% of window attention, which has no global pathway
    (\S\ref{subsec:results-systems}).

    \item \textbf{Evidence that separates storage order from geometry.} With the operator
    otherwise fixed, axial RoPE, packed-sequence RoPE, and a separable learned 3D embedding
    reach comparable Allen-15 Dice. A membership-preserving repack of the same volumes then
    costs packed RoPE 0.048 Dice on average, with a loss in every seed, and changes \method{}
    by 0.0001 (\S\ref{subsec:results-repacking}). Medical volumes, a six-seed comparison with a
    parameter-matched 3D U-Net, branch interventions, and constructed probes delimit the claim.
\end{enumerate}

\section{Related Work}
\label{sec:related}
\paragraph{Attention for 3D vision.}
Vision Transformers and subsequent scaling and self-supervised results showed that
attention-based visual backbones learn strong representations given sufficient data and
compute \citep{dosovitskiyImageWorth16x162021,zhaiScalingVisionTransformers2022,
dehghaniScalingVisionTransformers2023,heMaskedAutoencodersAre2022}. Attention couples local evidence with long-range context, which suits dense 3D prediction, where
local structure often depends on larger spatial organization.

\paragraph{Efficient and sparse attention.}
A large literature reduces the cost of attention by changing how dense attention is computed,
by approximating it, or by changing the attention graph. FlashAttention keeps exact dense
attention and improves its memory movement through IO-aware tiling
\citep{daoFlashAttentionFastMemoryEfficient2022}. Linformer, Performer, and Nystr\"omformer
approximate it through projections, random features, or landmarks
\citep{wangLinformerSelfAttentionLinear2020a,choromanskiRethinkingAttentionPerformers2020,
xiongNystromformerNystombasedAlgorithm2021}. Longformer and BigBird replace all-pairs attention
with structured local and sparse long-range graphs
\citep{beltagyLongformerLongDocumentTransformer2020a,zaheerBigBirdTransformers2021}. Native
Sparse Attention is closest in systems motivation: sparse attention should be trainable and
hardware-aligned, not merely sparse on paper \citep{yuanNativeSparseAttention2025b}. Packed 3D
fields add a geometric constraint. Beyond reducing the number of query--key pairs, the
positional mechanism must not treat memory order as sample geometry. Memory order is a systems
choice; attention geometry is a modeling choice.

\paragraph{Position encodings and geometry under packing.}
Self-attention alone is permutation equivariant, so transformers need positional information
\citep{vaswaniAttentionAllYou2017}. Absolute embeddings attach position to token features, and
relative methods make query--key compatibility a function of displacement
\citep{vaswaniAttentionAllYou2017,shawSelfAttentionRelativePosition2018}. RoPE does so through
query and key rotations \citep{suRoFormerEnhancedTransformer2023}, and axial and mixed
multidimensional variants have been studied for vision transformers \citep{heoRotaryVision2024}.
Coordinate-based models such as NeRF use explicit coordinates to represent geometry directly
\citep{mildenhallNeRFRepresentingScenes2020}. Axial RoPE on a token grid is therefore not our
contribution; we study how it interacts with compressed summaries and with changes of memory
order. A separable learned 3D embedding is also indexed by coordinates and can be
packing-equivariant, so we treat it as a substantive control.

\paragraph{Volumetric transformers and biomedical segmentation.}
U-Net, 3D U-Net, and V-Net established strong convolutional baselines for dense 3D prediction
\citep{ronnebergerUNetConvolutionalNetworks2015,cicek3DUNetLearning2016a,
milletariVNetFullyConvolutional2016}. CoTr, UNETR, Swin UNETR, and nnFormer combine transformer
context with convolutional feature extraction or U-shaped decoding
\citep{xieCoTrEfficientlyBridging2021,hatamizadehUNETRTransformers3D2022a,
hatamizadehSwinUNETRSwin2022,zhouNnFormerVolumetricMedical2023}. Our comparisons are of a
different kind: positional variants inside one operator, plus a parameter-matched 3D U-Net and
a UNETR control. We did not train Swin UNETR or nnFormer, so our results do not rank \method{}
among published volumetric segmenters.

\paragraph{Microscopy and medical volumes.}
Allen WTC--11 microscopy is a natural test bed because intracellular organization is explicitly
spatial: structure location, morphology, and inter-structure arrangement are the scientific
objects \citep{vianaIntegratedIntracellularOrganization2023}. Allen pipelines, including the
Allen Cell and Structure Segmenter, produce and analyze automated structure masks from
fluorescence images \citep{chenAllenCellStructure2020}; we predict those Allen-generated masks.
The Medical Segmentation Decathlon supplies a second volumetric domain
\citep{antonelliMedicalSegmentationDecathlon2022}.

\section{Methods}
\label{sec:method}
\method{} rests on three contracts. Local attention is exact, because microscopy carries
high-frequency morphology that summaries would erase. Context beyond the neighborhood is seen
only through summaries. And query--key compatibility is parameterized by 3D coordinates, never
by packed sequence index. This section makes each contract precise
(\S\ref{subsec:method-packed-fields}--\ref{subsec:method-local-global}), states what a
repacking must preserve for the operator to be unchanged (\S\ref{subsec:method-repacking}), and
describes how the operator is trained and run
(\S\ref{subsec:method-branch-mixing}--\ref{subsec:method-systems}).

\subsection{Packed token fields and coordinate geometry}
\label{subsec:method-packed-fields}

After patch embedding, the tokens occupy a lattice
\(\Omega=\{0,\ldots,Z{-}1\}\times\{0,\ldots,H{-}1\}\times\{0,\ldots,W{-}1\}\) with \(N=|\Omega|\)
sites; here \(Z\), \(H\), and \(W\) count tokens, not voxels. A packing order is a bijection
\(\pi:\{1,\ldots,N\}\to\Omega\) from memory positions to lattice sites, and token \(i\) carries
the coordinate \(c_i=\pi(i)=(z_i,y_i,x_i)\). In other words, the memory position \(i\) says
where a token is stored, and \(c_i\) says where it is. All reported models use the raw
token-grid coordinate; a spacing-aware map \((\Delta_z z_i,\Delta_y y_i,\Delta_x x_i)\) leaves every statement below
unchanged in form (Appendix~\ref{app:coordinate-scaling}). The only requirement is that \(c_i\)
describe the field, not the packed index.

The only positional mechanism is axial RoPE. Its per-head transform
\(R(c)=R_z(z)\oplus R_y(y)\oplus R_x(x)\) is a direct sum of planar rotations at fixed
frequencies, so \(R(c)^\top R(c')=R(c'-c)\). Queries and the keys of each branch
\(b\in\{\mathrm{loc},\mathrm{cmp}\}\) are rotated. Values are not, because RoPE defines a
geometry on query--key compatibility, not on content aggregation:
\[
    q_i^R = R(c_i)\,q_i,\qquad
    k_i^{b,R} = R(c_i)\,k_i^b,\qquad
    v_i^b = W_V^b\LN(h_i),
\]
where \(q_i=W_Q\LN(h_i)\) and \(k_i^b=W_K^b\LN(h_i)\) are the unrotated projections. Everything
below rests on the relative-displacement identity
\begin{equation}
    \langle R(c_i)q,\, R(c_j)k\rangle = q^\top R(c_j-c_i)\,k .
    \label{eq:main-relative-rope}
\end{equation}
A logit therefore depends on the displacement \(c_j-c_i\) and never on the packed positions
\(i\) and \(j\). Each 64-dimensional head has 32 rotary pairs, split
\((m_z,m_y,m_x)=(11,11,10)\) with base \(10{,}000\). The split is balanced even though the Allen
token grid, \(32\times128\times64\), is anisotropic; we did not tune it.

\subsection{Geometry-aware local--global factorization}
\label{subsec:method-local-global}

Each token update combines an exact local path with a compressed global path. The local branch
attends over a neighborhood \(\mathcal{N}(i)\) defined by coordinate offsets and validity, never
by adjacency in the packed array:
\begin{equation}
    o_i^{\mathrm{loc}} = \sum_{u\in\mathcal{N}(i)} \alpha_{iu}^{\mathrm{loc}}\, v_u^{\mathrm{loc}},
    \qquad
    \alpha_{iu}^{\mathrm{loc}}
    = \softmax_{u\in\mathcal{N}(i)}\!\Big(\tfrac{(q_i^R)^\top k_u^{\mathrm{loc},R}}{\sqrt{d_h}}\Big).
    \label{eq:local}
\end{equation}
All experiments use radius \((1,1,1)\), so \(|\mathcal{N}(i)|\le 27\). This branch is exact by
design. Boundaries, puncta, membranes, and thin structures are high-frequency signals, and
replacing nearby interactions by summaries would erase the detail that dense prediction needs.

The compressed branch summarizes the whole field in two stages. First, a partition
\(B_1,\ldots,B_F\) of the valid tokens into fine blocks yields one summary per block. Its weights
\(a_{ju}\) are nonnegative and sum to one within \(B_j\), and keys are rotated \emph{before}
they are pooled:
\begin{equation}
    \widehat{k}_j^R = \sum_{u\in B_j} a_{ju}\, R(c_u)\, k_u^{\mathrm{cmp}},
    \quad
    \widehat{v}_j = \sum_{u\in B_j} a_{ju}\, v_u^{\mathrm{cmp}},
    \quad
    a_{ju} = \softmax_{u\in B_j}\!\Big(\tfrac{w^\top k_u^{\mathrm{cmp},R}}{\tau} + b_{s_j(u)}\Big).
    \label{eq:fine-summary}
\end{equation}
The weights do not depend on the query, but they are learned and depend on content: the scorer
sees rotated keys, and \(b_{s_j(u)}\) is a bias for the token's canonical slot in its block.
Second, groups \(J_g\) of \(p\) fine summaries are averaged,
\(\widetilde{k}_g^R=|J_g|^{-1}\sum_{j\in J_g}\widehat{k}_j^R\) and likewise \(\widetilde{v}_g\),
into \(M\approx F/p\) final summaries. Every token attends to all of them:
\begin{equation}
    o_i^{\mathrm{cmp}} = \sum_{g=1}^{M} \beta_{ig}\,\widetilde{v}_g,
    \qquad
    \beta_{ig} = \softmax_{g}\!\Big(\tfrac{(q_i^R)^\top \widetilde{k}_g^R}{\sqrt{d_h}}\Big).
    \label{eq:global}
\end{equation}
On Allen-15, \(N=262{,}144\) tokens, fine blocks of 16 tokens, and \(p=16\) give \(F=16{,}384\)
fine and \(M=1{,}024\) final summaries per KV head. The following identity is what makes this
compression geometric rather than mere pooling.

\begin{proposition}[Relative geometry survives compression]
\label{prop:relative}
Let \(S_g=\bigcup_{j\in J_g}B_j\) and \(A_{gu}=a_{ju}/|J_g|\) for \(u\in B_j\), \(j\in J_g\). Then
\begin{equation}
    (q_i^R)^\top \widetilde{k}_g^R
    = \sum_{u\in S_g} A_{gu}\; q_i^\top R(c_u-c_i)\, k_u^{\mathrm{cmp}} .
    \label{eq:main-query-block-relative}
\end{equation}
\end{proposition}

\noindent
The proof is one line: substitute both pooling stages, use linearity, and
apply~\eqref{eq:main-relative-rope}; Theorem~\ref{thm:relative-summary} states it per KV head for
the realized coefficients. In words, each query-to-summary logit is a weighted sum of
query--key terms, each evaluated at a member token's displacement from the query. The identity
does not make compression lossless. Members are no longer individually addressable, softmax over
summaries is not softmax over tokens, and because the scorer sees rotated keys, the learned
weights need not be invariant to a common translation of all coordinates.

We rotate keys before compression because this keeps the identity exact. A cheaper summary-center variant, still coordinate-aware, pools unrotated keys and rotates the result once, at a representative
coordinate; Lemma~\ref{lem:center} bounds its logit error in terms of the rotary Lipschitz
constant and the support radius. On Allen-15 the two placements are within seed noise
(\S\ref{subsec:results-fixed}), so the empirical principle is that summaries should be
coordinate-aware. Only a probe in which token-level coordinates alone distinguish the inputs
separates the two placements (\S\ref{subsec:results-branch}).

\subsection{What a repacking must preserve}
\label{subsec:method-repacking}

A repacking permutes token rows by a matrix \(P\). For the operator to be unchanged up to that
permutation, more than coordinates must travel with the tokens.

\begin{proposition}[Packing equivariance]
\label{prop:packing}
Suppose a repacking carries with each token its coordinate and validity, its local
neighborhood edges, its fine-block and coarse-group memberships, and its canonical slot
within its block, and that all tokenwise maps are shared. Let \(P_S\) be the induced
permutation of final summaries. Then, in exact arithmetic,
\(\mathcal{A}(PX;P\mathcal{G}) = P\,\mathcal{A}(X;\mathcal{G})\), where \(\mathcal{G}\) collects
the coordinates and grouping metadata. Sequence-index RoPE has no such guarantee:
repacking changes each token's rotary coordinate.
\end{proposition}

\noindent
The proof (Theorem~\ref{thm:packing}) tracks the effective pooling matrix through both stages,
\(A\mapsto P_S A P^\top\). The score matrix becomes \(P S P_S^\top\), and row-softmax together with the permuted summary values
cancels \(P_S\), leaving the output permuted by \(P\). Reassigning slot biases to new array
offsets would break the argument. A learned 3D embedding looked up by coordinate also satisfies
the proposition, which is why it serves as a control. Together, Propositions~\ref{prop:relative}
and~\ref{prop:packing} are the core contract of \method{}: a membership-preserving repack only permutes the
computation, and each summary logit keeps its members' displacements.

Proposition~\ref{prop:packing} also separates the two interventions we test
(Figure~\ref{fig:repacking}a). A \emph{pure repack} permutes memory while preserving every
membership, so the proposition covers it. A \emph{named-layout switch} keeps coordinates and
neighborhoods but changes which tokens are summarized together; it changes the operator, and the
proposition is silent. The Allen models train under a raster-block layout, in which each nominal
\(1\times4\times4\) fine block is 16 consecutive raster tokens (a \(1\times1\times16\) strip) and
each final summary covers a \(1\times4\times64\) slab. The native block-Morton layout uses
coordinate-shaped \(1\times4\times4\) blocks in Morton order.

\subsection{Branch fusion and objective}
\label{subsec:method-branch-mixing}

In the Allen-15 and medical models the local output is kept as a residual, and the compressed
contribution is gated:
\[
    o_i = s_{\mathrm{loc}}\odot o_i^{\mathrm{loc}} + g_i\, s_{\mathrm{cmp}}\odot o_i^{\mathrm{cmp}},
\]
with channelwise scales \(s_{\mathrm{loc}},s_{\mathrm{cmp}}\) initialized at 1 and 0.25 and a
token- and head-wise softmax gate \(g_i\). Each token thus learns how much global context to
admit on top of exact local morphology. A second set of models, trained on six Allen structures
(Allen-6; \S\ref{subsec:results-contract}), instead uses the convex mixture
\(g_i^{\mathrm{loc}}o_i^{\mathrm{loc}}+g_i^{\mathrm{cmp}}o_i^{\mathrm{cmp}}\) with
\((g_i^{\mathrm{loc}},g_i^{\mathrm{cmp}})=\softmax(W_g\LN(h_i))\); the difference matters when we
intervene on a trained model (\S\ref{subsec:results-branch}). An output projection and a
feed-forward sublayer follow. All variants share one objective, cross-entropy plus foreground
soft Dice with equal weights, and every comparison holds the decoder, optimizer, schedule, crop
manifests, and input channels fixed.

\subsection{Systems realization and complexity}
\label{subsec:method-systems}

The systems goal is to make the mathematical sparse graph the executed training graph. The
reported runs use custom Triton \citep{tilletTritonIntermediateLanguage2019} forward and backward
kernels for local attention, driven by coordinate-defined neighbor indices. They stream
neighborhood K/V tiles instead of materializing an \(N\times w\times d_h\) gathered tensor, and
the backward pass recomputes local probabilities instead of saving a score tensor. Summary
construction runs as ordinary autograd operations, and query-to-summary attention uses PyTorch
scaled dot-product attention (SDPA), whose kernel and memory behavior depend on the selected
backend. The reported implementation is therefore not a single fused operator;
Appendix~\ref{app:streamed} describes an optional explicit-streaming backend that no reported result uses.

With neighborhood size \(w\), fine-block volume \(v\), and pooling factor \(p\), there are
\(M\approx N/(vp)\) summaries. Ignoring projections and normalization, dense attention costs
\(\Theta(H_q N^2 d_h)\), while \method{} costs
\[
    T = \underbrace{\Theta(H_q N w d_h)}_{\text{local}}
      + \underbrace{\Theta(H_q N M d_h)}_{\text{query--summary}}
      + \underbrace{\Theta(H_{kv} N d_h)}_{\text{summary build}} ,
\]
for \(H_q\) query heads, \(H_{kv}\) KV heads, and head dimension \(d_h\). The global term is the
dense cost divided by \(vp\). At a fixed compression ratio it is still quadratic in \(N\), so
\method{} is neither linear attention nor approximate local attention. RoPE changes the
coordinate system of the logits, not the attention graph.

\section{Results}
\label{sec:results}
The experiments test an operator claim through three questions, each with its own intervention.
Do positional variants differ on a fixed layout (\S\ref{subsec:results-fixed})? Does a trained
model depend on memory order (\S\ref{subsec:results-repacking})? Does it tolerate a new summary
partition (\S\ref{subsec:results-medical})?

\subsection{Experimental contract}
\label{subsec:results-contract}

\paragraph{Allen microscopy.}
Allen-15 predicts Allen-generated binary masks for 15 tagged structures from target
fluorescence, brightfield, and DNA; a learned embedding encodes structure identity. Each
\(32\times512\times256\) crop is patchified with \(1\times4\times4\) patches into a
\(32\times128\times64\) token grid (\(N=262{,}144\)), and the full grid is supervised and
scored. The depth-8 backbone has width 512 and eight query and two KV heads of dimension 64.
All positional arms share a token-grid U-Net decoder and train for 10k steps at batch size one,
with validation on all 120 validation crops every 1k steps. The Allen-6 study (six structures,
192/48 crops) differs in two ways that matter: it supervises only a central
\(16\times64\times64\) token region, with the surrounding collar input-only, and its mixer is
convex. We report it separately and never pool it with Allen-15
(Appendix~\ref{app:reproducibility}).

\paragraph{Medical volumes.}
MSD Heart (MRI) and Spleen (CT) are resampled isotropically to 1.25 and 1.5\,mm and cropped to
\(64^3\) with unit patches, again \(N=262{,}144\) \citep{antonelliMedicalSegmentationDecathlon2022}.
We train 18 depth-1 models from scratch (two tasks, three encodings, three seeds) under the
native block-Morton layout. Crops are centered using the provided foreground labels, so these
are controlled crop comparisons, not full-volume challenge scores. Validation uses 8 Heart and
16 Spleen crops from 4 and 8 held-out cases.

\paragraph{Selection and statistics.}
Seeds are 13, 29, and 41; six-seed comparisons add 53, 67, and 83. Best Dice is the best
validation checkpoint, and final Dice is the last. Layout interventions reuse the
original-layout best checkpoint without retraining or reselection. Tables report means and
sample standard deviations. We pair comparisons by seed and use two-sided exact sign-flip
tests, whose smallest attainable \(p\) is 0.25 with three pairs and 0.031 with six
(Appendix~\ref{app:statistics}). A nonsignificant difference is not evidence of equivalence.

\subsection{Fixed-layout controls}
\label{subsec:results-fixed}

\begin{table}[t]
\centering
\small
\caption{\textbf{On a fixed layout, matched positional and placement controls are within seed
noise.} Allen-15, 262k tokens, depth 8, 10k steps, three seeds; mean \(\pm\) sample s.d.
\(\Delta\) is the paired best-Dice difference from \method{}, with its exact sign-flip \(p\).
Each arm changes one factor; this is not a full factorial design. All arms instantiate 30.08M
parameters except separable 3D PE (30.20M).}
\label{tab:fixed-layout}
\begin{tabular}{lccc}
\toprule
Allen-15 variant & Best Dice & Final Dice & \(\Delta\) best (\(p\)) \\
\midrule
\method{} (axial RoPE, token-before) & 0.8228 \(\pm\) 0.0008 & 0.8228 \(\pm\) 0.0007 & --- \\
Packed-sequence RoPE                 & 0.8233 \(\pm\) 0.0032 & 0.8218 \(\pm\) 0.0059 & \(+0.0005\) (1.00) \\
Separable learned 3D PE              & 0.8190 \(\pm\) 0.0062 & 0.8185 \(\pm\) 0.0071 & \(-0.0038\) (0.50) \\
Summary-center RoPE                  & 0.8201 \(\pm\) 0.0063 & 0.8124 \(\pm\) 0.0176 & \(-0.0028\) (0.50) \\
Unrotated summary keys               & 0.8232 \(\pm\) 0.0040 & 0.8232 \(\pm\) 0.0040 & \(+0.0003\) (1.00) \\
Local-only, axial RoPE               & 0.8223 \(\pm\) 0.0036 & 0.8199 \(\pm\) 0.0068 & \(-0.0006\) (1.00) \\
\bottomrule
\end{tabular}
\end{table}

Table~\ref{tab:fixed-layout} varies one factor at a time inside the Allen-15 operator, and no
control establishes a reliable fixed-layout advantage. Packed-sequence RoPE reaches 0.8233 best
Dice against 0.8228 for axial RoPE; over six seeds the packed-minus-axial difference is
\(-0.0018\) (exact \(p=0.22\)). A separable learned 3D embedding of matched capacity trails by
0.0038 (\(p=0.50\)). Summary-center rotation, unrotated summary keys, and an independently
trained local-only model are likewise within seed noise. This parity is expected: on its
training layout, a sequence index is a consistent proxy for geometry (\S\ref{sec:intro}).
Section~\ref{subsec:results-repacking} tests what the parity hides.

The reviewed version's headline gap answers a different question. On Allen-6, the 24.3M-parameter \method{} reaches \(0.8531\pm0.0005\) best Dice against
\(0.7553\pm0.0042\) for a local--global model with a learned full-grid embedding and 158.5M
parameters. That baseline is not parameter-matched and does not isolate the positional
mechanism. The gap also decomposes: a separately trained axial local-only model already reaches
0.8387, so the compressed branch adds \(0.0144\), and the rest reflects a weak learned-position
baseline at this scale (Appendix~\ref{app:allen6}). Finally, Proposition~\ref{prop:packing}
makes memory order a free systems choice for a fixed partition, but not the choice of partition
itself. Trained from scratch under native block-Morton partitions instead of raster blocks, the
axial model reaches \(0.8232\pm0.0043\) against \(0.8228\pm0.0008\) (paired \(p=0.90\)).

\subsection{Pure repacking isolates memory-order dependence}
\label{subsec:results-repacking}

\begin{figure}[t]
    \centering
    \input{figures/fig_repacking.tex}
    \caption{\textbf{The repacking test.} \textbf{(a)}~Toy layouts of 64 tokens: fine blocks (thick
    white lines), four summary groups (shading), and memory order (strip). A pure repack moves
    whole groups, so coordinates, neighbors, and memberships travel with the tokens; a named switch
    regroups them. From query \(q\) to neighbor \(k\), the offset \(\Delta c=(\Delta y,\Delta x)\)
    never changes; the sequence offset \(\Delta\mathrm{seq}\) does. \textbf{(b)}~Allen-15 macro Dice
    (120 validation crops): seeds (markers), means (bars), and packed RoPE's mean paired change.
    The named switch uses three of the six seeds; dashed ticks mark those three on the training
    layout. \textbf{(c)}~Aligned hard predictions changed by the pure repack, on all validation
    crops. Separable PE was repacked on MSD only.}
    \label{fig:repacking}
\end{figure}

We now hold each trained checkpoint fixed and apply the pure repack of
Figure~\ref{fig:repacking}a. It moves 99.6\% of tokens in memory while preserving every
membership, so it lies within Proposition~\ref{prop:packing}. \method{} moves from 0.8247 to
0.8248 mean Dice across six seeds, whereas packed RoPE falls from 0.8229 to 0.7750
(Figure~\ref{fig:repacking}b). The packed-minus-axial change is \(-0.0480\), negative in all
six seed pairs (exact \(p=0.031\); paired \(t\)-test \(p=2\times10^{-4}\)). Aligned hard
predictions change at 0.0069\% of positions for \method{}, consistent with floating-point
reduction order, and at 1.23\% for packed RoPE (Figure~\ref{fig:repacking}c).

This is the identifying contrast. Two models with comparable accuracy on their training layout
respond differently when only the storage order changes. The sequence-indexed model has bound
part of what it learned to the layout; the coordinate-indexed model has not. One fixed
permutation is applied to every seed, so this is a stress test of storage-order dependence, not
an estimate over layouts.

\subsection{Named-layout changes and medical volumes}
\label{subsec:results-medical}

\begin{table}[t]
\centering
\small
\caption{\textbf{MSD at 262k tokens: comparable fixed-layout Dice, more prediction changes
for sequence indexing, small Dice consequences.} Depth 1, three seeds, all validation crops.
Named switch: native-trained checkpoints evaluated under the raster-block layout. \(\Delta\)Dice
in units of \(10^{-3}\).}
\label{tab:medical}
\begin{tabular}{llcccc}
\toprule
& & Fixed-layout & \multicolumn{2}{c}{Pure repack} & Named switch \\
\cmidrule(lr){4-5}\cmidrule(lr){6-6}
Task & Encoding & best Dice & Changed (\%) & \(\Delta\)Dice & \(\Delta\)Dice \\
\midrule
Heart  & \method{} (axial)  & 0.9422 \(\pm\) 0.0011 & 0.0038 & \(-0.01\) & \(-0.07\) \\
Heart  & Packed RoPE        & 0.9403 \(\pm\) 0.0010 & 0.1024 & \(+0.25\) & \(+0.06\) \\
Heart  & Separable 3D PE    & 0.9428 \(\pm\) 0.0005 & 0.0032 & \(-0.01\) & \(-0.02\) \\
\midrule
Spleen & \method{} (axial)  & 0.9567 \(\pm\) 0.0031 & 0.0051 & \(0.00\) & \(-0.15\) \\
Spleen & Packed RoPE        & 0.9610 \(\pm\) 0.0031 & 0.3184 & \(-0.28\) & \(-2.07\) \\
Spleen & Separable 3D PE    & 0.9531 \(\pm\) 0.0016 & 0.0034 & \(0.00\) & \(-1.02\) \\
\bottomrule
\end{tabular}
\end{table}

A named-layout switch also changes summary membership, so it tests robustness beyond
Proposition~\ref{prop:packing}. Evaluated under native block-Morton order, the raster-trained
Allen-15 checkpoints lose 0.0005 Dice for \method{} (\(0.8228\to0.8223\)) and 0.0788 for
packed RoPE (\(0.8233\to0.7445\); Figure~\ref{fig:repacking}b). The difference holds in all
three seed pairs (paired \(t\)-test \(p=0.0035\); exact \(p=0.25\)).

On MSD, all three encodings reach comparable fixed-layout Dice (Table~\ref{tab:medical});
separable PE leads numerically on Heart and packed RoPE on Spleen. Under the pure repack,
packed RoPE again changes more aligned predictions than both coordinate-indexed encodings in
every seed on both tasks (mean 0.10\% and 0.32\% of positions, against at most 0.0051\%), yet mean
Dice moves by less than 0.0003 for every arm. The named switch causes no consistent Heart
degradation and a modest Spleen loss for packed RoPE (mean 0.0021, with a loss in all three seeds; exact
\(p=0.25\)), which still leaves it with the highest post-switch Spleen Dice (0.9589).
Storage-order sensitivity is thus visible in a second domain, but whether it becomes accuracy
loss depends on the task: substantial on Allen, modest on Spleen, absent on Heart
(Appendix~\ref{app:medical} discusses candidate explanations).

\subsection{What the compressed branch contributes}
\label{subsec:results-branch}

An independently trained local-only model matches the fixed-layout Allen-15 score
(Table~\ref{tab:fixed-layout}), so natural-data Dice does not show that the task needs the
global branch. A jointly trained model can still depend on it. Zeroing the compressed output of
the three Allen-15 checkpoints at inference lowers Dice by 0.0094, and replacing every layer's
summaries with those of another sample lowers it by 0.0157, in all three seeds. Permuting
summary keys and values together changes Dice by only \(+0.0001\), as it must, since attention
over a set does not depend on the set's order (Appendix~\ref{app:interventions}). In the
convex-mixer Allen-6 models the same interventions cost 0.176 and 0.183, because there the
surviving local output keeps only its gate weight. The interventions establish learned dependence on the branch.

Natural-data Dice cannot isolate which channel a summary uses, so we built two controlled probes
(Appendix~\ref{app:probe-protocol}). In a placement probe, two token states are exchanged within
a block while the query, block center, token multiset, and uniformly pooled values stay fixed;
only token-level coordinates distinguish the pair. Token-before-compression reaches 0.9997
accuracy across six seeds, while summary-center and unrotated summaries stay at chance, 0.5000
(exact \(p=0.031\) for each contrast). In a long-range probe whose label lies outside the
query's radius-one neighborhood, the local--global model reaches 1.0000, while local-only
models of radius zero to two stay at 0.5000. Zeroing the remote evidence returns the full model
to chance, and flipping it inverts every prediction. Rotated summaries therefore carry a usable
token-coordinate channel, and the global branch delivers remote evidence when a task requires
it; neither probe shows that natural segmentation does.

\subsection{Architecture controls and fluorescence scope}
\label{subsec:results-arch}

\begin{table}[t]
\centering
\small
\caption{\textbf{Allen-15 architecture controls.} Same data, output grid, loss, and optimizer;
mean \(\pm\) sample s.d. The U-Net is closely parameter-matched. UNETR is an architecture
control, not a tuned benchmark: its \(16^3\) patching leaves 64 transformer tokens.}
\label{tab:architecture}
\begin{tabular}{lcccc}
\toprule
Model & Params & Seeds & Best Dice & Final Dice \\
\midrule
\method{}                         & 30.08M & 6 & \textbf{0.8247} \(\pm\) 0.0029 & \textbf{0.8228} \(\pm\) 0.0056 \\
TokenGrid 3D U-Net                & 31.87M & 6 & 0.8153 \(\pm\) 0.0030 & 0.8132 \(\pm\) 0.0050 \\
UNETR, 10k steps                  & 26.85M & 3 & 0.7430 \(\pm\) 0.0127 & 0.7430 \(\pm\) 0.0127 \\
UNETR, 20k steps                  & 26.85M & 3 & 0.7724 \(\pm\) 0.0023 & 0.7705 \(\pm\) 0.0045 \\
\bottomrule
\end{tabular}
\end{table}

Table~\ref{tab:architecture} compares \method{} with a closely parameter-matched TokenGrid 3D
U-Net trained identically. Over six seeds, \method{}'s best Dice is higher by 0.0094 in all six
pairs (exact \(p=0.031\); paired-\(t\) 95\% interval \([0.0065, 0.0123]\)). At the final
checkpoint the difference is 0.0096, with five of six wins (exact \(p=0.0625\)). \method{} has
the higher best-checkpoint mean on 13 of 15 structures (Wilcoxon \(p=0.010\)), and the U-Net is
ahead on CETN2 and NUP153. Few per-structure differences are significant: only ATP2A2 and SEC61B survive
a paired-\(t\) Holm correction (Appendix~\ref{app:architecture}). This is a modest, consistent
advantage in one matched comparison, not a ranking of attention above convolution.

The input includes target fluorescence, so we also test intensity alone. On Allen-6,
per-structure thresholds fitted on training crops reach 0.5576 Dice on held-out crops, against
0.8531 for \method{}. With half of the supervised-region target-fluorescence voxels dropped at
random, a depth-4 \method{} reaches 0.8355 and a U-Net 0.8415; brightfield-only models of both
kinds stay below 0.19. The task is therefore fluorescence-assisted segmentation of Allen-generated masks.

\subsection{Systems measurements}
\label{subsec:results-systems}

\begin{table}[t]
\centering
\footnotesize
\setlength{\tabcolsep}{3.8pt}
\caption{\textbf{Full-layer forward-plus-backward cost} under the trained Allen-15 attention
configuration (pooling factor 16, so \(M=N/256\)). One H200, batch one, bf16, median of ten
timed repetitions after three warmups; peak PyTorch allocation. Layer measurements, not
end-to-end training speedups.}
\label{tab:systems}
\begin{tabular}{rrrrrrrc}
\toprule
& & \multicolumn{3}{c}{Latency (ms)} & & & Peak GiB \\
\cmidrule(lr){3-5}
Tokens & \(M\) & Dense & Window & \method{} & Dense/\method{} & \method{}/Window & Dense / Window / \method{} \\
\midrule
65,536    & 256   & 103.8   & 13.86  & 16.14  & 6.43\(\times\)  & 1.16 & 1.19 / 2.04 / 1.45 \\
131,072   & 512   & 406.1   & 26.68  & 29.02  & 13.99\(\times\) & 1.09 & 2.32 / 4.02 / 2.82 \\
262,144   & 1,024 & 1,630.3 & 52.63  & 56.35  & 28.93\(\times\) & 1.07 & 4.57 / 7.98 / 5.58 \\
524,288   & 2,048 & 6,709.9 & 104.48 & 113.14 & 59.30\(\times\) & 1.08 & 9.08 / 15.89 / 11.09 \\
1,048,576 & 4,096 & 27,245.7 & 208.15 & 273.65 & 99.56\(\times\) & 1.31 & 18.10 / 31.72 / 22.11 \\
\bottomrule
\end{tabular}
\end{table}

Table~\ref{tab:systems} times one full attention layer under the configuration of the trained
Allen-15 models. Window attention is faster
than \method{} at every token count, by 7--9\% from 131k to 524k tokens and by 16\% and 31\% at
65k and 1M. The premium buys a global pathway that window attention lacks, and \method{} uses
less peak memory than window attention throughout. Against dense SDPA, \method{} is
\(6.43\times\) faster at 65k tokens and \(99.56\times\) faster at 1M, although dense SDPA has
the smallest peak allocation. Because \(M\) grows with \(N\) at a fixed pooling factor, the
global term stays quadratic. The claim is a bounded premium over window attention and large savings over dense attention.

\subsection{Claim boundaries}
\label{subsec:claim-boundaries}

The evidence does not support a fixed-layout accuracy advantage for axial RoPE, the necessity of
token-before-compression or of the global branch for natural-data Dice, or superiority over
volumetric segmenters at large. Three corrections to the version of this paper reviewed for
NeurIPS shape this reading (Appendix~\ref{app:corrections}). Its packed-sequence controls had no
active rotary pairs, its summary-shuffle ablation relied on an evaluator bug, and its threshold
baseline was tuned on validation crops. The corrected experiments reported here replace all
three.

\section{Conclusion}
\label{sec:conclusion}
Packed 3D attention fails when it treats storage order as geometry, and it fails silently. On
the layout a model was trained on, a sequence index can stand in for a coordinate, so
fixed-layout accuracy cannot reveal the dependence. \method{} indexes attention by coordinates,
including inside its compressed summaries, and Proposition~\ref{prop:packing} states what a
repacking must preserve for the operator to be unchanged. The result to remember is the
contrast of \S\ref{subsec:results-repacking}: two matched models with comparable Allen-15 Dice,
one of which loses 0.048 Dice when only the memory order of its input changes. The fix is a
positional swap inside an attention operator, so any model whose tokens carry physical
coordinates can adopt it, from volumetric segmentation to density fields and spatial
transcriptomics. A membership-preserving repack is a cheap test that such a pipeline can run
before trusting its positional encoding. Natural next steps are multi-atom summaries,
coordinate-respecting learned partitions, anisotropy-aware rotary budgets, and a streaming
kernel for additive 3D relative biases. The operator lesson is simple: memory layout is a
systems decision, not an attention coordinate system. Compressed global attention over packed 3D
fields should be built in the coordinates of the sample, not the sequence used to store it.

\paragraph{Limitations.}
The medical models are shallow, with small validation populations and crops centered on known
foreground. They show layout sensitivity with little accuracy consequence, and our experiments
do not explain the difference from Allen. Best checkpoints and reported scores share one
validation set. The U-Net margin is modest, Swin UNETR and nnFormer were not run, and the rotary
allocation for anisotropic grids was not ablated. The global work remains quadratic at a fixed
compression ratio, and window attention is faster per layer. These are research models, not
validated for clinical or scientific decisions; automated masks and label-guided crops can
propagate bias into downstream measurements.

\begin{ack}
We thank Matheus Viana and Suraj Mishra from the Allen Institute for Cell Science for helping us understand the WTC-11 hiPSC single-cell image dataset and how to curate it, and the Institute for making it publicly available. We also thank the organizers of the Medical Segmentation Decathlon for releasing the Heart and Spleen datasets. This material is based upon work supported by the Army Research Laboratory under cooperative agreement \texttt{W911NF2520159} and National Science Foundation under Grant No. \texttt{DGE-2125644} and \texttt{2411453}. 
\end{ack}

\bibliographystyle{plainnat}
\bibliography{references}

\clearpage
\appendix

\section{Experimental Contract and Claim Boundary}
\label{app:claim-status}

This appendix fixes the experimental contracts behind Section~\ref{sec:results} and states
where the claims stop. The central contribution is a packed-geometric-attention principle,
that physical coordinates rather than sequence order should index packed 3D attention,
instantiated by a local--global operator and tested by matched controls and repacking
interventions on microscopy and medical volumes. Nothing here broadens that claim to
segmentation state of the art, label-free prediction, fixed-layout superiority of one
positional encoding, faster-than-window execution, or practical multi-million-token
throughput. Appendix~\ref{app:theory} gives the theory, Appendix~\ref{app:additional-results}
the supporting results and the corrections to the reviewed version, and
Appendix~\ref{app:selected-systems} the systems details and optional extensions.

\subsection{Reproducibility contract}
\label{app:reproducibility}

Table~\ref{tab:app-contracts} records the three training contracts. They differ in ways that
matter for interpretation, so their results are never pooled.

\begin{table}[!htbp]
\centering
\small
\caption{Training contracts. All token counts are 262{,}144. Allen splits are disjoint by
field of view and crop within each structure; plate or well overlap is possible, so they are
not claimed to be plate-disjoint. MSD splits are disjoint by case.}
\label{tab:app-contracts}
\begin{tabularx}{\linewidth}{lYYY}
\toprule
Item & Allen-15 (principal) & Allen-6 (reviewed version) & MSD Heart / Spleen \\
\midrule
Role & Positional controls, repacking, architecture & Capacity, collar, and intervention results & Second domain \\
Targets & 15 tagged structures, Allen-generated masks & TOMM20, LMNB1, ST6GAL1, GJA1, LAMP1, ATP2A2 & Left atrium (MRI); spleen (CT) \\
Train / val.\ crops & 1{,}440 / 120 & 192 / 48 & 64 / 8; 132 / 16 \\
Input & \(32\times512\times256\) voxels; fluorescence, brightfield, DNA & same & \(64^3\) at 1.25 / 1.5\,mm; one channel \\
Patch; token grid & \(1\times4\times4\); \(32\times128\times64\) & same & \(1\times1\times1\); \(64^3\) \\
Supervised region & Full token grid & Central \(16\times64\times64\) tokens; collar input-only & Full token grid \\
Backbone & Depth 8, width 512, 8 query / 2 KV heads, \(d_h=64\) & same & Depth 1, otherwise same \\
Training layout & \texttt{raster\_block} & \texttt{raster\_block} & \texttt{native\_block\_morton} \\
Fine block; pooling & 16 tokens; factor 16 (\(M=1{,}024\)) & same & \(2\times2\times2\); factor 16 (\(M=2{,}048\)) \\
Fusion & Residual local, gated global & Convex softmax mixture & Residual local, gated global \\
Rotary & Axial, \((m_z,m_y,m_x)=(11,11,10)\), base \(10^4\), raw grid coordinates & same & same \\
Local radius & \((1,1,1)\) & \((1,1,1)\) & \((1,1,1)\) \\
Seeds & 13, 29, 41; +53, 67, 83 & 13, 29, 41 & 13, 29, 41 \\
\bottomrule
\end{tabularx}
\end{table}

\subsection{Data and supervision}
\label{app:data}

\paragraph{Allen.}
Allen-15 contains 96 training and eight validation crops per structure for AAVS1, ACTN1,
ATP2A2, CETN2, DSP, FBL, HIST1H2BJ, LAMP1, LMNB1, MYH10, NPM1, NUP153, SEC61B, TOMM20, and
TUBA1B. The manifest also lists 60 test crops, which no reported comparison uses. Each input
channel is clipped to its 0.2th--99.8th percentiles and standardized. Target fluorescence,
brightfield, and DNA are inputs; no segmentation-derived channel is included, and structure
identity is supplied as a learned embedding before the backbone rather than inferred from the
masks. Targets are Allen-generated binary masks, not independent manual annotations. Labels
are downsampled to the token grid by nearest-neighbor interpolation, and loss and Dice are
computed on that grid, not at raw voxel resolution.

In the Allen-6 study, only the central \(16\times256\times256\) voxel region, i.e.\ the central
\(16\times64\times64\) target tokens, contributes to the loss and to evaluation. The surrounding
context is input-only and never supervised, so the 262{,}144 tokens are not all loss-bearing.
The 192/48 crops are distinct field-of-view-aligned crops, not replicas of one sample.

\paragraph{Medical.}
Images and labels are brought into a common canonical orientation, resampled (images
linearly, labels by nearest neighbor) to 1.25\,mm isotropic spacing for Heart and 1.5\,mm for
Spleen, transposed to \((z,y,x)\) order, clipped to the 0.5th--99.5th intensity percentiles,
standardized per volume, and zero-padded to at least \(64^3\). Crop centers are sampled
deterministically from the provided foreground labels, with training-only jitter of up to a
quarter crop width. Validation therefore assumes a known foreground location. The fixed case
split has 16/4 Heart and 33/8 Spleen training/validation cases, with four training and two
validation crops per case; the 8 and 16 validation crops are not 8 and 16 independent
patients. A preprocessing audit finds no case or crop overlap between training and validation.

\subsection{Architecture and optimization}
\label{app:architecture-details}

The backbone uses pre-normalized attention and a residual feed-forward block, separate
branch-specific K/V projections, and a learned summary scorer with one atom. Invalid neighbors
are masked. Selected-token refinement (Appendix~\ref{app:selected-refinement}) and atom
regularizers are disabled in every reported run. The backbone's spatial output is projected
to 32 channels and decoded by a MONAI BasicUNet with feature widths
\((32,32,64,128,256,32)\) and transposed-convolution upsampling, without a raw-input skip; the
medical models have no structure embedding. Training uses AdamW with learning rate
\(2\times10^{-4}\), weight decay 0.01, batch size one, and bf16 mixed precision for 10{,}000
updates, validating every 1{,}000. The objective is cross-entropy plus a foreground soft
Dice loss with equal weights, where for foreground probabilities \(p_v\) and binary targets
\(y_v\) the Dice term is \(1-(2\sum_v p_vy_v+1)/(\sum_v p_v+\sum_v y_v+1)\). Hard predictions
are the argmax over background and foreground. Allen macro Dice averages crop scores within
each structure and then averages structures equally; medical Dice averages the fixed
validation crops.

\subsection{Positional controls}
\label{app:positional-controls}

\emph{Packed-sequence RoPE} replaces each token's coordinate by its packed sequence index and
assigns all 32 rotary pairs to that index; its runtime transform was checked to be nonzero.
\emph{Separable learned 3D PE} adds \(E_z(z)+E_y(y)+E_x(x)\) to the token features before the
backbone and uses no Q/K rotation; its coordinate lookup is preserved under repacking.
\emph{Summary-center} and \emph{unrotated-summary-key} controls keep the axial local branch and
change only how global keys are rotated. Because the fine weights are scored from the keys
they receive, these trained placement arms also change the scorer's inputs; the uniform-pooling
probe of Appendix~\ref{app:probe-protocol} removes that confound. \emph{Local-only} disables
the compressed branch. Every arm otherwise shares the configuration, data manifest, decoder,
loss, optimizer, and budget of \method{}.

\subsection{Layout interventions}
\label{app:layout-interventions}

The pure repack permutes intact coarse groups, each of 16 fine blocks, and preserves fine and
coarse membership, local edges, valid tokens, and within-block slots. Predictions are restored
to the original grid before decoding and comparison. The Allen permutation uses generator
seed 2026072600 and the medical permutation 20260726; each is a fixed stress test applied to
every training seed, not a sample from a distribution of layouts. The named switch rebuilds the
grouping metadata for the destination layout (Allen: raster to native block-Morton; medical:
native block-Morton to raster), which changes every fine and coarse membership while keeping
coordinates and local neighborhoods. Prediction disagreement is the number of unequal,
spatially aligned hard predictions divided by the number of evaluated token-grid positions.
Neither intervention retrains or reselects a checkpoint.

\subsection{Uncertainty and selection}
\label{app:statistics}

For paired differences \(d_1,\ldots,d_n\) across training seeds, the two-sided exact sign-flip
statistic is
\[
    p = 2^{-n}\sum_{s\in\{-1,1\}^n}
    \mathbf{1}\Big\{\big|\tfrac1n\textstyle\sum_i s_i d_i\big| \ge \big|\tfrac1n\textstyle\sum_i d_i\big|\Big\}.
\]
Its null assumes sign-exchangeable paired differences, not independent voxels. With three
nonzero same-direction pairs the smallest attainable value is 0.25; with six it is 0.03125.
Paired-\(t\) intervals use \(n-1\) degrees of freedom and additionally assume roughly normal
differences. Per-structure tests apply a Holm correction across the 15 structures,
separately for each checkpoint horizon; the Wilcoxon signed-rank test in
\S\ref{subsec:results-arch} is computed across the 15 per-structure mean differences. The
six-seed extensions followed the three-seed analyses and were not preregistered. Best
checkpoints and reported scores share the validation set, so final-checkpoint results are
reported as a selection sensitivity. Training-seed variability says nothing about
patient-level or cell-level uncertainty.

\subsection{Classical threshold baselines}
\label{app:threshold-baselines}

Because the input includes the target fluorescence channel, the first claim-boundary question
is whether the Allen task collapses to thresholding. It does not.
Table~\ref{tab:app-threshold-baselines} reports classical target-fluorescence controls on the
Allen-6 validation crops. The tuned threshold is fitted per structure on training crops and
scored on held-out validation crops; the reviewed version fitted it on the validation crops
it was scored on (0.5582), and the held-out protocol changes the result only slightly. The
higher-recall classical rows over-predict foreground and miss the Allen-generated masks.

\begin{table}[!htbp]
\centering
\small
\caption{Classical target-fluorescence controls on the 48 Allen-6 validation crops, against
0.8531 for \method{}. The tuned threshold is fitted on training crops only. FG: foreground.}
\label{tab:app-threshold-baselines}
\begin{tabular}{lccc}
\toprule
Method & Macro Dice & FG recall & Predicted FG fraction \\
\midrule
Per-structure threshold (training-fitted) & 0.5576 & 0.7224 & 0.0896 \\
Otsu per crop                     & 0.5021 & 0.8108 & 0.1239 \\
Otsu + connected-component cleanup & 0.5004 & 0.7882 & 0.1197 \\
Gaussian smoothing + Otsu         & 0.3565 & 0.8763 & 0.2079 \\
Adaptive local threshold          & 0.2171 & 0.8967 & 0.3802 \\
\bottomrule
\end{tabular}
\end{table}

\subsection{Constructed mechanism probes}
\label{app:probe-protocol}

Both probes use one attention layer of width 48 with three query and three KV heads of
dimension 16, one frozen uniform summary atom, the native block-Morton layout, balanced axial
frequencies, and reference PyTorch attention rather than the high-token Triton kernels. Training uses
AdamW (learning rate 0.003, weight decay \(10^{-4}\), gradient clipping at norm one) on batches
of 32 counterfactual pairs; evaluation uses 512 pairs per seed, for each of the six seeds.

The \emph{placement probe} uses a \(4^3\) lattice with \(2^3\) blocks, reads the query at
\((0,0,0)\), trains for 150 updates, and exposes only the compressed branch to the readout. Two
token states are exchanged within a block while the query, coordinate set, block center, token
multiset, and uniformly pooled values stay fixed, so summary-center and unrotated summaries
cannot distinguish the pair through that channel. The \emph{long-range probe} uses a \(6^3\)
lattice, a query at \((1,1,1)\), and an evidence block with origin \((4,4,4)\), trained for 120
updates. A one-layer local-only model of radius zero, one, or two cannot reach the label bit;
radius three can. Zeroing and flipping the remote bit test whether the trained global model
actually uses it. Both constructions establish that an information channel is available and
used; neither establishes that natural data requires it. Figure~\ref{fig:app-probes} draws both
constructions with their results.

\begin{figure}[!htbp]
\centering
\input{figures/fig_probes}
\caption{The two constructed probes, drawn projected along \(z\). (a)~Placement: the two
inputs differ only by a swap of two token states inside one block (which block is schematic).
Only token-before rotation pools a key that changes under the swap, and only it beats chance.
(b)~Long-range: the label bit lies in a block three sites from the query along each axis, so a
local window of radius two or less cannot see it. The local--global model reads it through a
summary; zeroing the bit returns the model to chance, and flipping it inverts every prediction.
Bars are means over six seeds: blue marks \method{}'s choice in each probe (token-before rotation
in (a), local plus global attention in (b)), light blue that model under an intervention, grey the
controls, and the dashed vertical line chance. Neither probe shows that
natural segmentation needs these channels.}
\label{fig:app-probes}
\end{figure}

\subsection{Data access and responsible use}
\label{app:assets}

Allen microscopy is obtained from the Allen Cell Explorer
\citep{vianaIntegratedIntracellularOrganization2023,chenAllenCellStructure2020} under its
stated data terms.\footnote{\url{https://www.allencell.org/terms-of-use.html}} The Heart and
Spleen descriptors specify CC-BY-SA 4.0 \citep{antonelliMedicalSegmentationDecathlon2022}, and
NoduleMNIST3D is distributed under CC BY 4.0 \citep{yangMedMNIST2023}. MONAI is Apache-2.0,
Triton is MIT-licensed, and PyTorch uses a BSD-style license. We collect no new patient data or
crowdsourced labels, and we do not redistribute any scans. The models are research tools, not
validated diagnostic systems: automated labels, label-guided crop selection, and small
populations can bias downstream measurements, and unvalidated use in medical decisions is a
potential harm.

\FloatBarrier
\section{Mathematical Details and Theory for Geometry-Aware Packed 3D Attention}
\label{app:theory}

This appendix proves the results of Section~\ref{sec:method} and records what they do not
say. One distinction controls everything: in a packed 3D field every token carries two indices,
a memory position and a physical coordinate, and a correct attention operator must depend only
on the second. We fix notation in \S\ref{app:setup} and recall in \S\ref{app:rope-representation}
why axial RoPE turns coordinates into relative displacements (Lemma~\ref{lem:relative},
Corollary~\ref{cor:translation}). We then prove that the full two-stage operator is
packing-equivariant (Theorem~\ref{thm:packing}), that sequence-index RoPE is not
(Proposition~\ref{prop:sequence}), and that relative geometry survives compression
(Theorem~\ref{thm:relative-summary}). Section~\ref{app:selection} shows that coordinate-aware summaries separate inputs that
coordinate-blind pooling cannot (Theorem~\ref{thm:selection}), and Sections~\ref{app:rope-placement}
and~\ref{app:center-bound} bound the one approximation the method admits
(Lemmas~\ref{lem:center} and~\ref{lem:softmax}). Throughout we separate what is \emph{exactly}
invariant to repacking from what is only \emph{approximately} coordinate-aware. No statement
asserts lossless compression or a downstream accuracy advantage.

\subsection{Setup and notation}
\label{app:setup}

Let \(\Omega\subset\mathbb{Z}^3\) be a finite token lattice with \(N=|\Omega|\), and let a packing
be a bijection \(\pi:\{1,\ldots,N\}\to\Omega\). The sequence index \(i\) is a memory position;
\(x=\pi(i)\) is the physical lattice site. A fixed coordinate map \(\kappa:\Omega\to\mathbb{R}^3\)
supplies the geometry the operator may see, with \(c_x=\kappa(x)\) and \(c_i=c_{\pi(i)}\). Every
result below holds for any fixed \(\kappa\), including the anisotropic and normalized maps of
Appendix~\ref{app:coordinate-scaling}; the reported models use the raw token-grid
coordinate. Features are written \(q_x,k_x^b,v_x^b\) for branch \(b\in\{\mathrm{loc},\mathrm{cmp}\}\),
and a packed array lists site-indexed quantities in the order \(\pi\) induces.

Local neighborhoods \(\mathcal{N}(x)\subset\Omega\) are defined by coordinate offsets and
validity. For each KV head, let \(A_f\in\mathbb{R}^{F\times N}\) hold the fine-block weights
\(a_{ju}\), zero outside block \(B_j\), and let \(C\in\mathbb{R}^{M\times F}\) hold the uniform
coarse weights \(C_{gj}=1/|J_g|\) for \(j\in J_g\). Then \(A=CA_f\in\mathbb{R}^{M\times N}\) is the
effective token-to-summary map, with entries \(A_{gu}\) as in Proposition~\ref{prop:relative}.
Every retained row of \(A_f\), \(C\), and \(A\) is nonnegative and sums to one. Empty blocks
are omitted, invalid members carry zero weight, and every local row is nonempty because the
neighborhoods include the query itself.

Under a repacking \(\pi'\) of the same field, features and coordinates permute together: with
\(P\) the induced token permutation, \(X^{\pi'}=PX^{\pi}\), and the coordinate array is permuted
by the same \(P\). We write
\(P_F\) and \(P_S\) for the induced enumerations of fine blocks and final summaries (both are the
identity if blocks are enumerated canonically from coordinates). An operator \(\mathcal{A}\) is
\emph{packing-equivariant} if \(\mathcal{A}(PX;P\mathcal{G})=P\mathcal{A}(X;\mathcal{G})\), where
\(\mathcal{G}\) collects coordinates, validity, neighborhoods, and grouping metadata. This is the
right invariance for packed fields: repacking memory may permute the outputs, but it must not
change the function computed on the underlying geometry. Locality-improving layouts do not
supply it by themselves: a Morton index places nearby tokens near each other in memory, but it
is still a layout, not a metric, and two packings of the same field assign different indices
to the same token.

\subsection{Axial RoPE as an orthogonal representation of displacements}
\label{app:rope-representation}

Let \(d_h\) be the per-head dimension and
\[
    R(c) = \bigoplus_{a\in\{z,y,x\}}\ \bigoplus_{r=0}^{m_a-1}
    \begin{bmatrix}
        \cos(\omega_{a,r}c_a) & -\sin(\omega_{a,r}c_a) \\
        \sin(\omega_{a,r}c_a) & \phantom{-}\cos(\omega_{a,r}c_a)
    \end{bmatrix}
    \ \oplus\ I_{\mathrm{rem}} .
\]
The reported segmentation models allocate all 64 head channels to 32 pairs, balanced as
\((m_z,m_y,m_x)=(11,11,10)\) for axial RoPE and \((32,0,0)\) for the corrected sequence-index
control, which uses the packed index as its single coordinate. The proofs need only that \(R\)
is an \emph{orthogonal representation of additive displacements},
\begin{equation}
    R(c)^\top R(c') = R(c'-c),\qquad R(0)=I .
    \label{eq:rope-rep}
\end{equation}

\begin{lemma}[Relative-displacement identity]
\label{lem:relative}
\emph{(An exact identity.)} For any \(c,c'\in\mathbb{R}^3\) and \(q,k\in\mathbb{R}^{d_h}\),
\(\langle R(c)q,\,R(c')k\rangle = q^\top R(c'-c)\,k\).
\end{lemma}

\begin{proof}
A planar rotation satisfies \(R(\theta)^\top R(\theta')=R(\theta'-\theta)\). Taking direct sums
across axes and frequencies gives~\eqref{eq:rope-rep}; substituting into the inner product
proves the identity.
\end{proof}

\begin{corollary}[Translation of a token-pair logit]
\label{cor:translation}
\emph{(A symmetry of pairwise logits.)} Holding unrotated features fixed, translating both
coordinates by a common \(t\) leaves the token-pair rotary logit unchanged.
\end{corollary}

\begin{proof}
\((c'+t)-(c+t)=c'-c\), so Lemma~\ref{lem:relative} applies.
\end{proof}

The corollary concerns a pairwise compatibility, not every operation that uses rotated
features. For example, \(w^\top R(c_u+t)k_u\) generally differs from \(w^\top R(c_u)k_u\), so
the learned pooling weights of~\eqref{eq:fine-summary} need not be translation-invariant, and
translating a finite crop need not preserve its contents, neighborhoods, or summary supports.

\subsection{Packing equivariance of the full two-stage operator}
\label{app:packing-theorem}

\begin{theorem}[Packing equivariance of coordinate-indexed local--global attention]
\label{thm:packing}
\emph{(An equivariance theorem.)} Let a permutation \(P\) of token rows transport
coordinates, features, validity, and local edges, and let \(P_F,P_S\) be the induced
enumerations of fine blocks and final summaries. Assume that
\begin{enumerate}
    \item every fine-block membership and every coarse-group membership is preserved;
    \item canonical within-block slot identities, and any slot parameters, travel with their
    members, so re-enumerating a block re-enumerates its weight vector by the same
    permutation rather than assigning weights by new array offsets; and
    \item shared tokenwise maps and their parameters are unchanged, and stochastic operations
    are disabled or their random variables are transported.
\end{enumerate}
Then the deterministic exact-arithmetic operator of
Equations~\eqref{eq:local}--\eqref{eq:global} with residual-local, gated-global fusion is
packing-equivariant. The conclusion also holds for the convex tokenwise mixer.
\end{theorem}

\noindent
The idea is to follow the permutation through each stage of the operator. Every
intermediate array has its token rows permuted by \(P\), or its block or summary rows permuted
by \(P_F\) or \(P_S\); shared pointwise maps commute with these permutations, and row-wise softmax
commutes with the row and column permutations of a score matrix. Figure~\ref{fig:app-chase}
draws the chase.

\begin{figure}[!htbp]
\centering
\begin{tikzpicture}[>={Stealth[length=3.6pt,width=3pt]}, font=\small,
  lbl/.style={font=\footnotesize, text=cInk, inner sep=2pt},
  why/.style={font=\footnotesize, text=cInk, align=left, anchor=west, inner sep=0pt},
  ref/.style={font=\footnotesize, text=cMuted},
  arr/.style={->, draw=cInk, line width=0.45pt}]
\matrix (m) [matrix of math nodes, row sep=5.2mm, column sep=15mm,
             nodes={anchor=center, inner sep=2.4pt, text height=1.75ex, text depth=0.55ex}] {
X                              & PX \\
(Q,\,K,\,V)                    & (PQ,\,PK,\,PV) \\
(\widehat K,\,\widehat V)      & (P_F\widehat K,\,P_F\widehat V) \\
(\widetilde K,\,\widetilde V)  & (P_S\widetilde K,\,P_S\widetilde V) \\
S                              & PSP_S^{\top} \\
O^{\mathrm{cmp}}               & PO^{\mathrm{cmp}} \\
};
\node[ref, anchor=south] at ($(m-1-1.north)+(0,1.2mm)$) {packing \(\pi\)};
\node[ref, anchor=south] at ($(m-1-2.north)+(0,1.2mm)$) {repacking \(\pi'\)};
\draw[arr] (m-1-1) -- node[above, lbl]{\(P\)} (m-1-2);
\draw[arr] (m-2-1) -- node[above, lbl]{\(P\)} (m-2-2);
\draw[arr] (m-3-1) -- node[above, lbl]{\(P_F\)} (m-3-2);
\draw[arr] (m-4-1) -- node[above, lbl]{\(P_S\)} (m-4-2);
\draw[arr] (m-5-1) -- node[above, lbl]{\(P(\cdot)P_S^{\top}\)} (m-5-2);
\draw[arr] (m-6-1) -- node[above, lbl]{\(P\)} (m-6-2);
\draw[arr] (m-1-1) -- node[left, lbl]{shared maps, \(R(c_i)\)} (m-2-1);
\draw[arr] (m-2-1) -- node[left, lbl]{fine pooling \(A_f\)} (m-3-1);
\draw[arr] (m-3-1) -- node[left, lbl]{coarse pooling \(C\)} (m-4-1);
\draw[arr] (m-4-1) -- node[left, lbl]{scores \(Q\widetilde K^{\top}\!/\sqrt{d_h}\)} (m-5-1);
\draw[arr] (m-5-1) -- node[left, lbl]{\(\softmax_{\mathrm{row}}(\cdot)\,\widetilde V\)} (m-6-1);
\foreach \r [evaluate=\r as \s using int(\r+1)] in {1,...,5} {
  \draw[arr] (m-\r-2) -- (m-\s-2);
}
\coordinate (W) at ($(m-4-2.east)+(4mm,0)$);
\node[why] at ($(W|-m-1-2)!0.5!(W|-m-2-2)$) {the permutation carries the coordinates,\\ and the maps are shared (3)};
\node[why] at ($(W|-m-2-2)!0.5!(W|-m-3-2)$) {\(A_f'=P_FA_fP^{\top}\): fine memberships\\ kept, slots transported (1, 2)};
\node[why] at ($(W|-m-3-2)!0.5!(W|-m-4-2)$) {\(C'=P_SCP_F^{\top}\): coarse memberships\\ kept (1), uniform coarse weights};
\node[why] at ($(W|-m-4-2)!0.5!(W|-m-5-2)$) {using \(Q'=PQ\):\\ \((PQ)(P_S\widetilde K)^{\top}\!/\sqrt{d_h}=PSP_S^{\top}\)};
\node[why] at ($(W|-m-5-2)!0.5!(W|-m-6-2)$) {using \(\widetilde V'=P_S\widetilde V\): row softmax commutes\\ with the permutations, and \(P_S^{\top}P_S=I\)};
\end{tikzpicture}
\caption{The proof of Theorem~\ref{thm:packing} as a diagram chase through the compressed
branch, whose keys and values are \(K\) and \(V\). Rows are the stages of the operator. Each
horizontal arrow is the permutation that relates the two arrays: \(P\) acts on token rows,
\(P_F\) on fine-block rows, and \(P_S\) on summary rows. Each square commutes for the reason at
its right, where (1)--(3) are the theorem's assumptions, so the output is permuted by \(P\). The
local branch is the same chase with \(L'=PLP^{\top}\), and the fusion maps are shared tokenwise
maps.}
\label{fig:app-chase}
\end{figure}

\begin{proof}
Shared projections and coordinate-indexed rotations give \(Q'=PQ\), \(K^{b\prime}=PK^b\), and
\(V^{b\prime}=PV^b\), where \(Q\) and \(K^b\) denote rotated arrays. The fine weights are a
shared function of each member's rotated key and its transported slot, so assumptions~1
and~2 give \(A_f'=P_FA_fP^\top\); uniform coarse groups give \(C'=P_SCP_F^\top\); hence
\(A'=C'A_f'=P_SAP^\top\). Summaries then permute by block,
\(\widetilde{K}'=A'K^{\mathrm{cmp}\prime}=P_S\widetilde{K}\) and \(\widetilde{V}'=P_S\widetilde{V}\),
and the scaled score matrix obeys \(S'=PSP_S^\top\). Row-wise softmax commutes with these
row and column permutations, so
\[
    O^{\mathrm{cmp}\prime} = \softmax(PSP_S^\top)\,P_S\widetilde{V}
    = P\,\softmax(S)\,P_S^\top P_S\widetilde{V} = P\,O^{\mathrm{cmp}} .
\]
For the local branch, the masked score matrix transforms as \(L'=PLP^\top\) because both
coordinates and neighborhood edges travel with the sites, and the same argument gives
\(O^{\mathrm{loc}\prime}=PO^{\mathrm{loc}}\). Gates, channelwise scales, output projections,
residual maps, and tokenwise MLPs are shared pointwise maps and commute with \(P\). Composing
these steps proves the claim, for either mixer.
\end{proof}

A within-block softmax of a shared content score satisfies assumption~2 when any canonical slot
bias is transported: the weight \emph{vector} is equivariant to member enumeration, and its
weighted sum is invariant. Content dependence is allowed. An order-sensitive recurrence over
members, or a slot bias reassigned to new memory offsets, would violate the assumption, as
would rebuilding blocks or coarse groups from new consecutive memory intervals; the
named-layout switches of \S\ref{subsec:results-medical} are of this last kind and are not
theorem checks. Grouped-query attention applies the argument to each KV group separately. A
decoder that operates on the restored spatial grid receives the same field and need not
operate on packed arrays. Finite-precision reduction order, backend selection, and stochastic
masks can change actual outputs under repacking; numerical tests check an implementation
within tolerances but neither prove the theorem nor guarantee identical training trajectories.

\subsection{Why sequence position is not a geometric coordinate}
\label{app:sequence}

\begin{proposition}[Sequence-index RoPE is not packing-equivariant]
\label{prop:sequence}
\emph{(An impossibility, by explicit counterexample.)} There exist two packings of one field
for which the same physical pair receives different sequence-index rotary logits.
\end{proposition}

\begin{proof}
Take three physical tokens \(a,b,c\) under packings \((a,b,c)\) and \((a,c,b)\). The physical pair
\((a,b)\) has index gap one under the first and two under the second. In one rotary pair of
frequency \(\omega\) with \(q=e_1\) and \(k=e_2\), the logit at gap \(\Delta\) is
\(q^\top R(\omega\Delta)k=-\sin(\omega\Delta)\). For \(\omega=1\), \(\sin 1\neq\sin 2\), so the pair
receives different logits; restoring the physical row order does not remove a discrepancy
attached to a fixed physical pair.
\end{proof}

The proposition is an existence statement. It does not say that every repacking changes every
trained prediction (zero features, phase coincidences, or a task that ignores the affected
channel can mask the difference), nor that sequence order is useless to a kernel. It says that
sequence order is not, on its own, a spatial coordinate on a repacked field; coordinate
indexing earns the invariance of Theorem~\ref{thm:packing} that memory indexing cannot. On a
single fixed layout, by contrast, a sequence index can serve as a consistent proxy for
geometry, which is why fixed-layout accuracy cannot reveal the difference
(\S\ref{subsec:results-fixed}).

\subsection{Relative geometry survives block compression}
\label{app:relative-summary}

\begin{theorem}[Relative-logit identity for final summaries]
\label{thm:relative-summary}
\emph{(An exact identity.)} For \(S_g=\bigcup_{j\in J_g}B_j\) and \(A_{gu}=a_{ju}/|J_g|\), the final
summary key satisfies, for every query site \(x\),
\[
    (R(c_x)q_x)^\top\widetilde{k}_g^R
    = \sum_{u\in S_g} A_{gu}\, q_x^\top R(c_u-c_x)\, k_u^{\mathrm{cmp}} .
\]
It holds for the realized scalar coefficients, including content-dependent ones, and for each
KV head separately.
\end{theorem}

\begin{proof}
The two pooling stages give \(\widetilde{k}_g^R=\sum_{u\in S_g}A_{gu}R(c_u)k_u^{\mathrm{cmp}}\).
Distribute the inner product over this sum and apply Lemma~\ref{lem:relative}.
\end{proof}

This is the exact reason the method rotates token keys \emph{before} pooling them: each query-to-summary logit is a weighted sum of query--key terms, each evaluated at the
displacement of a token that formed the summary; it is never a similarity to a position-less
pooled key. Linearity
alone needs only scalar coefficients; nonnegativity and normalization give the averaging
interpretation and the bounds of \S\ref{app:center-bound}; query independence makes one
summary reusable by every query. Values are pooled without rotation. Softmax of a pooled logit
is not the average of token-level softmaxes, so the identity does not turn compressed
attention into dense attention.

\subsection{Coordinate-awareness is strictly more expressive}
\label{app:selection}

The previous identity is not cosmetic: it buys representational power that content-only
pooling lacks. Consider a compressed-only readout with two singleton summaries at
\(+\delta\) and \(-\delta\) relative to a query. They have identical unrotated keys \(k\) and
scalar values \(+1\) and \(-1\), and the target is the value at \(+\delta\). A
\emph{coordinate-blind} comparator computes its logits from the query and these identical keys
alone, with no block identity, coordinate, value-dependent score, or other readout path.

\begin{theorem}[Relative-spatial selection]
\label{thm:selection}
\emph{(A separation.)} The coordinate-blind comparator outputs zero, with squared error one. A
rotary-coordinate comparator approaches \(+1\) arbitrarily closely if query norms are not
bounded by a fixed constant.
\end{theorem}

\begin{proof}
Identical keys give equal logits and weights \(\tfrac12\), so the coordinate-blind output is
\(\tfrac12(1)+\tfrac12(-1)=0\). For one rotary pair choose \(q=-e_1\), \(k=e_2\), and \(\delta\) with
\(\sin(\omega\delta)>0\); then \(\Delta=q^\top R(\delta)k-q^\top R(-\delta)k>0\). Scaling \(q\) by
\(t>0\) gives output \(\tanh(t\Delta/2)\), which tends to one as \(t\to\infty\); the usual
\(1/\sqrt{d_h}\) scaling is absorbed into \(t\).
\end{proof}

Three scope notes keep this result from being over-read. It is an architectural separation:
the construction uses unbounded logit scaling and a hand-chosen rotary pair, so it is exact
only in the limit and says nothing about what a trained network does. It separates
coordinate-aware from coordinate-blind selection, not token-before from summary-center
rotation, since rotating a singleton summary at its center also solves it. And we have no
natural-data example in which Allen structures require this selection rule; the closest
controlled evidence is the placement probe of \S\ref{subsec:results-branch}, which separates
token-level from center-level coordinates under a different construction.

\subsection{RoPE placement in the compressed branch: exact versus approximate}
\label{app:rope-placement}

Theorem~\ref{thm:relative-summary} is exact for token-before-compression. The summary-center
variant pools unrotated keys and rotates the pool once at a representative coordinate, so its
logit approximates, but does not equal, the exact one. We are explicit that the two differ in
theoretical status while their natural-data scores do not separate: on Allen-15,
summary-center rotation reaches \(0.8201\pm0.0063\) best Dice against \(0.8228\pm0.0008\)
(exact \(p=0.50\); Table~\ref{tab:fixed-layout}), and an exploratory single-seed gate likewise found them close (Table~\ref{tab:app-exploratory}). We
therefore make no claim that token-before-compression is empirically necessary for
segmentation. It is the canonical form because it alone preserves the token-level identity,
and the placement probe confirms that this identity carries information the center variant
cannot. The next subsection bounds how far the approximation can drift.

\subsection{How far the summary-center approximation can drift}
\label{app:center-bound}

Fix a final support \(S\) with normalized scalar weights \(A_u\) and a center \(\mu\), and
compare the exact key \(\sum_u A_uR(c_u)k_u\) with \(R(\mu)\bar{k}\), where
\(\bar{k}=\sum_uA_uk_u\). All weights and values are held fixed; retraining with a different
pooling rule is a different experiment.

\begin{lemma}[Center-rotation error]
\label{lem:center}
\emph{(A deterministic bound in block geometry.)} If \(\|q\|_2\le Q\), \(\|k_u\|_2\le K\), and
\(\rho=\max_u\|c_u-\mu\|_2\), the scaled logit error satisfies
\[
    |\ell-\widehat{\ell}|
    \le \frac{Q}{\sqrt{d_h}}\sum_u A_u\,\|R(c_u)-R(\mu)\|_{\mathrm{op}}\,\|k_u\|_2
    \le \frac{QK}{\sqrt{d_h}}\min\{2,\,L_R\rho\},
\]
where \(L_R=\max_{a,r}|\omega_{a,r}|\). If also \(\|k_u-\bar{k}\|_2\le\epsilon_K\), the looser bound \(Q(\epsilon_K+KL_R\rho)/\sqrt{d_h}\) of the reviewed version follows.
\end{lemma}

\begin{proof}
Subtracting the two summary keys before taking the inner product gives
\[
    \ell-\widehat{\ell}=d_h^{-1/2}\,(R(c_x)q)^\top\sum_uA_u\big(R(c_u)-R(\mu)\big)k_u .
\]
Orthogonality of \(R(c_x)\) and the triangle inequality give the first bound. For a planar rotation, the operator
norm of the difference between angles \(\theta\) and \(\theta'\) is
\(2|\sin((\theta-\theta')/2)|\le\min\{2,|\theta-\theta'|\}\); for a direct sum it is the
maximum over pairs, at most \(\min\{2,L_R\|c_u-\mu\|_2\}\). Averaging gives the second bound,
and the weaker bound follows because \(\epsilon_K\ge0\).
\end{proof}

The bound is informative only when \(L_R\rho\) is small, and the realized supports make that
explicit. With base \(10^4\) the highest rotary frequency is one radian per grid unit, so
\(L_R=1\). In the raster contract the final support is a \(1\times4\times64\) slab, whose radius
about its center is \(\rho=\sqrt{1.5^2+31.5^2}\approx31.5\); the bound then saturates at the
trivial value \(2QK/\sqrt{d_h}\) and says nothing. Applied pair by pair, the same argument
shows that the error is driven by high-frequency pairs: a pair of frequency \(\omega\)
contributes at most \(\min\{2,\omega\rho\}\), which is small only when \(\omega\ll1/\rho\approx0.03\).
For a single coordinate-shaped \(1\times4\times4\) fine block, \(\rho\approx2.1\) and the bound
becomes informative for pairs below about half a radian per unit. The lemma therefore predicts
when the two placements \emph{can} differ, at high frequencies and on large supports; it is not
a guarantee that center rotation performs poorly, and the trained comparison above shows that
on Allen-15 it does not. Figure~\ref{fig:app-center} draws the two placements and the two
supports.

\begin{figure}[!htbp]
\centering
\begin{tikzpicture}[x=1cm, y=1cm, >={Stealth[length=3.2pt,width=2.6pt]},
  lab/.style={font=\footnotesize, text=cInk, inner sep=0pt},
  mut/.style={font=\footnotesize, text=cMuted, inner sep=0pt},
  ttl/.style={font=\footnotesize, text=cInk, inner sep=0pt, anchor=base west}]
\def\a{0.30}
\node[ttl] at (0,2.30) {\textbf{(a)}\enspace Two placements of the rotation};
\fill[shB] (2.00,0.40) rectangle +(4*\a,4*\a);
\foreach \i in {1,2,3} {
  \draw[white, line width=0.5pt] (2.00+\i*\a,0.40) -- +(0,4*\a);
  \draw[white, line width=0.5pt] (2.00,0.40+\i*\a) -- +(4*\a,0);
}
\draw[cInk, line width=0.5pt] (2.00,0.40) rectangle +(4*\a,4*\a);
\coordinate (mu) at (2.60,1.00);
\coordinate (x) at (0.25,1.00);
\foreach \u in {(2.15,0.55),(3.05,0.55),(2.15,1.45),(3.05,1.45)} {
  \draw[->, cAxial, line width=0.6pt, shorten >=1.1pt] (x) -- \u;
  \fill[cAxial] \u circle (0.9pt);
}
\draw[->, cMuted, line width=0.9pt, dash pattern=on 2.2pt off 1.4pt, shorten >=2.2pt] (x) -- (mu);
\draw[cInk, line width=0.6pt] ($(mu)+(-0.06,-0.06)$) -- ($(mu)+(0.06,0.06)$) ($(mu)+(-0.06,0.06)$) -- ($(mu)+(0.06,-0.06)$);
\draw[cInk, line width=0.45pt] (mu) -- (3.05,1.45);
\node[lab, anchor=south west] at (3.22,1.47) {\(\rho\)};
\node[lab, anchor=north] at (2.60,0.30) {support \(S\), center \(\mu\)};
\fill[cInk] (x) circle (1.5pt);
\node[lab, anchor=east] at ($(x)+(-0.08,0)$) {\(c_x\)};
\draw[cAxial, line width=0.6pt] (0.05,-0.20) -- +(0.40,0);
\node[lab, anchor=west] at (0.55,-0.20) {exact: rotate each member at \(c_u\)};
\draw[cMuted, line width=0.9pt, dash pattern=on 2.2pt off 1.4pt] (0.05,-0.55) -- +(0.40,0);
\node[lab, anchor=west] at (0.55,-0.55) {summary-center: rotate the pool at \(\mu\)};
\def\t{0.115}
\def\X{5.30}
\node[ttl] at (\X,2.30) {\textbf{(b)}\enspace Two supports, to scale (one \(z\)-slice)};
\fill[shB] (\X,1.10) rectangle +(64*\t,4*\t);
\foreach \k in {1,2,3} { \draw[white, line width=0.7pt] (\X+16*\k*\t,1.10) -- +(0,4*\t); }
\foreach \j in {1,2,3} { \draw[white, line width=0.7pt] (\X,1.10+\j*\t) -- +(64*\t,0); }
\draw[cInk, line width=0.5pt] (\X,1.10) rectangle +(64*\t,4*\t);
\coordinate (m1) at (\X+32*\t,1.10+2*\t);
\draw[cInk, line width=0.45pt] (m1) -- (\X+63.5*\t,1.10+3.5*\t);
\draw[cInk, line width=0.6pt] ($(m1)+(-0.06,-0.06)$) -- ($(m1)+(0.06,0.06)$) ($(m1)+(-0.06,0.06)$) -- ($(m1)+(0.06,-0.06)$);
\node[lab, anchor=base west] at (\X,1.76) {raster-block final summary: \(1\times4\times64\) slab, 16 strips of \(1\times1\times16\)};
\node[lab, anchor=north east] at (\X+64*\t,1.02) {\(\rho\approx31.5\)};
\fill[shB] (\X,0.10) rectangle +(4*\t,4*\t);
\foreach \j in {1,2,3} {
  \draw[white, line width=0.4pt] (\X+\j*\t,0.10) -- +(0,4*\t);
  \draw[white, line width=0.4pt] (\X,0.10+\j*\t) -- +(4*\t,0);
}
\draw[cInk, line width=0.5pt] (\X,0.10) rectangle +(4*\t,4*\t);
\coordinate (m2) at (\X+2*\t,0.10+2*\t);
\draw[cInk, line width=0.45pt] (m2) -- (\X+3.5*\t,0.10+3.5*\t);
\node[lab, anchor=west] at (\X+4*\t+0.15,0.34) {a single coordinate-shaped fine block, \(1\times4\times4\): \(\rho\approx2.1\)};
\node[mut, anchor=base west, align=left] at (\X,-0.50) {the per-pair form of the bound, \(\min\{2,\omega\rho\}\), is informative only for\\ \(\omega\lesssim1/\rho\): about 0.03 (slab) and 0.5 (block) radians per grid unit};
\end{tikzpicture}
\caption{Exact versus summary-center rotation. (a)~Token-before compression rotates each member
key at its own coordinate \(c_u\), so the logit sees every member's displacement from the query
(four of sixteen are drawn); summary-center rotation applies one rotation at the center \(\mu\).
Lemma~\ref{lem:center} bounds the difference through the support radius \(\rho\). (b)~The final
support of the raster-block models is a \(1\times4\times64\) slab with \(\rho\approx31.5\); for
comparison, a single coordinate-shaped \(1\times4\times4\) fine block has \(\rho\approx2.1\). The bound marks where
the two placements can differ; on Allen-15 their Dice scores do not separate
(Table~\ref{tab:fixed-layout}).}
\label{fig:app-center}
\end{figure}

\begin{lemma}[Softmax and output stability]
\label{lem:softmax}
\emph{(A Lipschitz bound from logits to outputs.)} If \(\|\ell-\widehat{\ell}\|_\infty\le\epsilon\)
and the same summary values satisfy \(\|\widetilde{v}_g\|_2\le V\), then
\(\|\softmax(\ell)-\softmax(\widehat{\ell})\|_1\le 2\epsilon\) and
\(\|o-\widehat{o}\|_2\le 2V\epsilon\).
\end{lemma}

\begin{proof}
For \(s=\softmax(\ell)\) the Jacobian is \(J=\operatorname{diag}(s)-ss^\top\), and for any \(u\),
\(\|Ju\|_1\le\|u\|_\infty\sum_{ij}|J_{ij}|=2(1-\|s\|_2^2)\|u\|_\infty\le2\|u\|_\infty\). Integrating
the Jacobian along the segment between the logits proves the first inequality; the triangle
inequality applied to the value-weighted probability difference proves the second.
\end{proof}

Both inequalities can also be capped at 2 and \(2V\). If a placement change alters the pooling
coefficients it alters the summary values too, and the corresponding value error must be
added; the fixed-coefficient bound alone does not control a separately trained model. Composing
Lemmas~\ref{lem:center} and~\ref{lem:softmax} shows that the summary-center output error is
controlled by support radius, rotary frequency, and key norms. This is not a claim that
microscopy is smooth; it is the formal reason \method{} keeps nearby interactions exact instead of
relying on summaries for them.

\subsection{Coordinate scaling and anisotropy}
\label{app:coordinate-scaling}

Because every result above holds for an arbitrary fixed \(\kappa\), anisotropic and normalized
coordinates need no further work. For a site \(u=(u_z,u_y,u_x)\) one may use physical spacing,
\(c_u=(\Delta_zu_z,\Delta_yu_y,\Delta_xu_x)\), or normalized coordinates with denominators
\(\max(1,Z{-}1)\), \(\max(1,H{-}1)\), \(\max(1,W{-}1)\). Changing \(\kappa\) changes the rotary
phases, and so defines a different modeling choice, but none of the algebra. The reported
models use raw token-grid coordinates with the balanced allocation of
\S\ref{subsec:method-packed-fields}; the isotropic resampling of the medical volumes does not
make the Allen coordinates physical, and we claim no frequency-allocation or physical-spacing
ablation.

\subsection{Attention-graph complexity}
\label{app:complexity}

The reduction is in the graph, not in the coordinate system. Dense attention computes \(N^2\)
token--token interactions. With neighborhood size \(w\), fine-block volume \(v\), pooling factor
\(p\), and \(M\approx N/(vp)\), \method{} computes \(Nw\) exact local interactions, \(NM\approx
N^2/(vp)\) query--summary interactions, and \(O(N)\) summary-construction terms. The nonlocal
graph changes from token--token to token--summary, cutting global pairings by the factor
\(vp\). The backward pass differentiates the same interactions with the same asymptotic
structure; recomputing scores changes saved state and constants, not the \(NM\) term. The
operator is linear in \(N\) only if \(M\) is held fixed as \(N\) grows, which the reported
configuration does not do; at a fixed ratio the global work is reduced quadratic.

\subsection{Comparison with additive 3D relative biases}
\label{app:positional-alternatives}

A coordinate-aware logit can also be built additively, \(L_{xg}=q_x^\top\widetilde{k}_g+b(c_x,\mu_g)\),
through a 3D relative-position or ALiBi-style bias. This is a legitimate alternative, not a
strawman, and the difference we draw is one of systems contract rather than principle. At
\(N=262{,}144\) and \(M=1{,}024\), with \(H_{kv}=2\) and \(H_q=8\), the dense additive-bias
implementation available to us would create \(H_{kv}NM=536{,}870{,}912\) bias elements and
\(H_qNM=2{,}147{,}483{,}648\) grouped-score elements per layer, about 5\,GiB in bf16 before
gradients and 40\,GiB across eight layers; its runtime guard stops it at this scale. (The reviewed version's count of 4.29 billion elements per head group used the fine-summary count
\(16{,}384\) in place of the final count \(1{,}024\).) A systems-matched comparison therefore needs
a streamed bias-aware kernel, which we have not built, and we claim no principal-scale
accuracy comparison with relative biases. Axial RoPE avoids the tensor: by~\eqref{eq:rope-rep}
it injects relative geometry through query/key rotation and reuses the ordinary dot-product
path. Single-seed comparisons with 3D relative-position and ALiBi-style biases at 65k tokens,
from the reviewed version, are listed in Table~\ref{tab:app-exploratory}.

\subsection{Summary of the theory}
\label{app:theory-summary}

The theory supports four claims and declines a fifth. First, coordinate-indexed local--global
attention is a well-defined operator on a packed 3D field: a repacking that transports
coordinates, neighborhoods, memberships, and slots only permutes the output. Second,
sequence-index positional attention is not a geometric operator on packed fields, because
changing the memory layout changes its positional coordinate. Third, rotating keys before
compression makes every query-to-summary logit a weighted sum of content interactions
modulated by query-relative 3D displacement, through both pooling stages. Fourth,
coordinate-aware compression is strictly more expressive than coordinate-blind compression
for relative-spatial selection. The theory does not claim that any of these properties raises
accuracy on a fixed layout; Section~\ref{sec:results} shows where they matter empirically,
which is under changes of storage order.

\FloatBarrier
\section{Diagnostics, Supporting Results, and Claim Boundaries}
\label{app:additional-results}

This section follows the evidence chain in three movements: diagnostics that confirm the
implementation computes the operator the theory describes; supporting tables for the results
of Section~\ref{sec:results}, with adverse comparisons reported rather than smoothed; and the
boundary material, including the exploratory gates and the corrections to the reviewed
version. All natural-data scores are validation scores, and ``best'' always refers to the
original training layout.

\subsection{Geometry diagnostics: does the implementation match the theory?}
\label{app:geometry-diagnostics}

Before any task-level number is interpreted, finite-example checks confirm that the executed
operator satisfies the identities of Section~\ref{sec:method}. Table~\ref{tab:app-geometry-tests}
groups the 29 CPU checks run against the project implementation, which include its actual
RoPE, learned fine compression, and coarse pooling, together with nine standalone checks of the
two-stage transport and of the two bounds (250 random trials each). Tests are not proofs; they
show that the code realizes the theorems within tolerance rather than merely resembling them.

\begin{table}[!htbp]
\centering
\small
\caption{Implementation diagnostics. Exact identities pass at float64 round-off (at most
\(1.8\times10^{-15}\)); the implementation's fp32 relative identity passes at
\(3.0\times10^{-8}\). Counterexample checks pass when the expected violation is observed.}
\label{tab:app-geometry-tests}
\begin{tabularx}{\linewidth}{lY}
\toprule
Check & Verified property \\
\midrule
Rotation algebra & Each \(R(c)\) is orthogonal and \(R(c)^\top R(c')=R(c'-c)\); all-pair logits depend only on displacement. \\
Compression identity & Rotated-key pooling equals the weighted relative-displacement sum, through learned fine weights with slot biases and through the second pooling stage, for keys and for (unrotated) values. \\
Packing transport & Two-stage weights transport as \(A'=P_SAP^\top\); summaries, scores, global outputs, local outputs, and local-plus-gated-global outputs are permutation-equivariant; a common summary K/V permutation is a no-op. \\
Gradients & Attention gradients with respect to \(q\), \(k\), and \(v\) match analytic values. \\
Negative witnesses & An untransported slot bias, a changed coarse membership, a sequence-index repack, and recomputed content weights under translation each break the corresponding property, as Appendix~\ref{app:theory} predicts. \\
Scope checks & A pooled logit does not reproduce dense block attention mass; the center and softmax/output bounds hold in all trials, including zero radius. \\
\bottomrule
\end{tabularx}
\end{table}

\subsection{Allen-15 repacking and fixed-layout controls by seed}
\label{app:allen-seeds}

Table~\ref{tab:app-repacking-seeds} gives the six paired trajectories behind
Figure~\ref{fig:repacking}b. On the original layout the packed-minus-axial difference is
\(-0.0018\) (exact \(p=0.22\)); the paired difference of repacking changes is \(-0.0480\) (exact
\(p=0.031\)). The repack thus isolates storage-order dependence without relying on a
fixed-layout accuracy gap. Table~\ref{tab:app-factor-seeds} gives the three-seed fixed-layout
controls of Table~\ref{tab:fixed-layout} by seed.

\begin{table}[!htbp]
\centering
\small
\caption{Allen-15 pure repacking by training seed: unchanged best checkpoints, all 120
validation crops. ``Changed'' is the percentage of aligned hard predictions that differ.}
\label{tab:app-repacking-seeds}
\begin{tabular}{rcccccc}
\toprule
& \multicolumn{3}{c}{\method{} (axial RoPE)} & \multicolumn{3}{c}{Packed-sequence RoPE} \\
\cmidrule(lr){2-4}\cmidrule(lr){5-7}
Seed & Original & Repacked & Changed (\%) & Original & Repacked & Changed (\%) \\
\midrule
13 & 0.8232 & 0.8233 & 0.0060 & 0.8253 & 0.7731 & 1.121 \\
29 & 0.8234 & 0.8234 & 0.0083 & 0.8196 & 0.7535 & 1.261 \\
41 & 0.8219 & 0.8223 & 0.0075 & 0.8250 & 0.7723 & 1.414 \\
53 & 0.8239 & 0.8240 & 0.0061 & 0.8186 & 0.7800 & 1.151 \\
67 & 0.8301 & 0.8302 & 0.0063 & 0.8264 & 0.7801 & 1.262 \\
83 & 0.8255 & 0.8256 & 0.0070 & 0.8223 & 0.7907 & 1.150 \\
\midrule
Mean & 0.8247 & 0.8248 & 0.0069 & 0.8229 & 0.7750 & 1.226 \\
\bottomrule
\end{tabular}
\end{table}

\begin{table}[!htbp]
\centering
\small
\caption{Allen-15 fixed-layout controls by seed (best / final Dice), depth 8, 10k updates.
The local-only arm is trained independently.}
\label{tab:app-factor-seeds}
\begin{tabular}{lccc}
\toprule
Variant & Seed 13 & Seed 29 & Seed 41 \\
\midrule
\method{} (axial, token-before) & 0.8232 / 0.8232 & 0.8234 / 0.8232 & 0.8219 / 0.8219 \\
Packed-sequence RoPE            & 0.8253 / 0.8253 & 0.8196 / 0.8149 & 0.8250 / 0.8250 \\
Separable learned 3D PE         & 0.8129 / 0.8112 & 0.8190 / 0.8190 & 0.8252 / 0.8252 \\
Summary-center RoPE             & 0.8160 / 0.7929 & 0.8169 / 0.8169 & 0.8273 / 0.8273 \\
Unrotated summary keys          & 0.8277 / 0.8277 & 0.8206 / 0.8206 & 0.8212 / 0.8212 \\
Local-only, axial RoPE          & 0.8200 / 0.8128 & 0.8205 / 0.8205 & 0.8264 / 0.8264 \\
\bottomrule
\end{tabular}
\end{table}

\paragraph{Training under two partitions.}
A separate three-seed axial comparison trains from scratch under each named layout.
Raster-trained models reach \(0.8228\pm0.0008\) best Dice and native-block-Morton-trained models
\(0.8232\pm0.0043\); the native-minus-raster difference is \(+0.0004\) (paired \(t\)-test
\(p=0.90\), exact \(p=1.0\)). This small unresolved difference is not an equivalence test, and it
answers a different question from evaluating one unchanged checkpoint under a new layout
(\S\ref{subsec:results-medical}).

\subsection{Medical results at both checkpoint horizons}
\label{app:medical}

Table~\ref{tab:app-medical-seeds} lists all 18 medical runs with both inference-only
interventions. Best and final scores differ substantially in some Spleen runs: axial final Dice
is \(0.9379\pm0.0243\) and separable final Dice \(0.9117\pm0.0285\), against \(0.9566\pm0.0033\)
for packed RoPE. Reporting only best scores would conceal this optimization sensitivity; the
data identify no unique cause and no universal stability ordering among the encodings. The
axial and packed models have 8.83M parameters and the separable models 8.92M.

\begin{table}[!htbp]
\centering
\small
\caption{All 18 medical runs. Both interventions reuse the original best checkpoint and score
the complete fixed validation set; changes are in units of \(10^{-3}\) Dice.}
\label{tab:app-medical-seeds}
\begin{tabular}{llrccccc}
\toprule
Task & Encoding & Seed & Best & Final & Pure: changed (\%) & Pure \(\Delta\) & Named \(\Delta\) \\
\midrule
Heart & \method{} (axial) & 13 & 0.9427 & 0.9398 & 0.0032 & \(-0.02\) & \(+0.18\) \\
Heart & \method{} (axial) & 29 & 0.9429 & 0.9390 & 0.0045 & \(-0.01\) & \(-0.38\) \\
Heart & \method{} (axial) & 41 & 0.9409 & 0.9409 & 0.0035 & \(-0.01\) & \(+0.01\) \\
Heart & Packed RoPE & 13 & 0.9391 & 0.9391 & 0.0846 & \(+0.31\) & \(-0.32\) \\
Heart & Packed RoPE & 29 & 0.9409 & 0.9390 & 0.1155 & \(+0.33\) & \(-0.10\) \\
Heart & Packed RoPE & 41 & 0.9407 & 0.9371 & 0.1071 & \(+0.11\) & \(+0.59\) \\
Heart & Separable 3D PE & 13 & 0.9428 & 0.9327 & 0.0036 & \(-0.02\) & \(-0.03\) \\
Heart & Separable 3D PE & 29 & 0.9433 & 0.9416 & 0.0023 & \(-0.04\) & \(-0.04\) \\
Heart & Separable 3D PE & 41 & 0.9423 & 0.9398 & 0.0038 & \(+0.02\) & \(+0.02\) \\
\midrule
Spleen & \method{} (axial) & 13 & 0.9568 & 0.9421 & 0.0045 & \(0.00\) & \(+0.68\) \\
Spleen & \method{} (axial) & 29 & 0.9598 & 0.9598 & 0.0080 & \(-0.02\) & \(-0.04\) \\
Spleen & \method{} (axial) & 41 & 0.9536 & 0.9117 & 0.0028 & \(0.00\) & \(-1.08\) \\
Spleen & Packed RoPE & 13 & 0.9643 & 0.9528 & 0.3493 & \(-0.83\) & \(-2.99\) \\
Spleen & Packed RoPE & 29 & 0.9581 & 0.9581 & 0.3667 & \(+0.20\) & \(-0.26\) \\
Spleen & Packed RoPE & 41 & 0.9605 & 0.9588 & 0.2391 & \(-0.20\) & \(-2.96\) \\
Spleen & Separable 3D PE & 13 & 0.9530 & 0.9006 & 0.0028 & \(+0.01\) & \(-4.19\) \\
Spleen & Separable 3D PE & 29 & 0.9515 & 0.9441 & 0.0035 & \(0.00\) & \(-0.01\) \\
Spleen & Separable 3D PE & 41 & 0.9547 & 0.8905 & 0.0040 & \(0.00\) & \(+1.13\) \\
\bottomrule
\end{tabular}
\end{table}

Our evidence for storage-order dependence is the controlled intervention, not a correlation
between packed and spatial distances. Model depth, foreground size, decoder smoothing, and
crop selection are all plausible contributors to the contrast between Allen and the medical
tasks; a tiny disagreement fraction can move macro Dice more on small structures than on
large organs. The present experiments do not separate these explanations.

\subsection{Architecture controls across seeds and structures}
\label{app:architecture}

Tables~\ref{tab:app-unet-seeds} and~\ref{tab:app-structures} expand the six-seed U-Net
comparison of Table~\ref{tab:architecture}. At the best checkpoint, \method{} has the higher
mean on 13 of 15 structures and at least four seed wins on 10; CETN2 and NUP153 favor the
U-Net. Only ATP2A2 and SEC61B survive the paired-\(t\) Holm correction, and no per-structure
exact sign-flip test survives a 15-way correction. At the final checkpoint the macro margin is
similar (0.0096; five of six seeds), 12 of 15 structure means favor \method{} (Wilcoxon
\(p=0.026\)), and only SEC61B survives Holm correction. The aggregate margin is consistent
across seeds; it does not imply per-structure gains everywhere.
Figure~\ref{fig:app-structures} plots the per-seed differences behind
Table~\ref{tab:app-structures}.

\begin{figure}[!htbp]
\centering
\input{figures/fig_structures}
\caption{Per-structure \method{} minus U-Net Dice on Allen-15, six seeds, ordered by the
best-checkpoint mean. Dots are seeds and bars their means; the dashed line is the macro
difference. At the best checkpoint 13 of 15 means are positive, and at the final checkpoint 12.
An asterisk marks a Holm-adjusted paired-\(t\) value below 0.05: ATP2A2 and SEC61B at the best
checkpoint, SEC61B at the final. The macro difference is positive in all six seeds at the best
checkpoint and in five of six at the final; the per-structure gains are not uniform.}
\label{fig:app-structures}
\end{figure}

\begin{table}[!htbp]
\centering
\small
\caption{Allen-15 architecture comparison by seed, 10k updates. ``Best'' is
validation-selected, not independently test-selected.}
\label{tab:app-unet-seeds}
\begin{tabular}{rcccccc}
\toprule
& \multicolumn{3}{c}{Best Dice} & \multicolumn{3}{c}{Final Dice} \\
\cmidrule(lr){2-4}\cmidrule(lr){5-7}
Seed & \method{} & U-Net & \(\Delta\) & \method{} & U-Net & \(\Delta\) \\
\midrule
13 & 0.8232 & 0.8172 & \(+0.0060\) & 0.8232 & 0.8048 & \(+0.0184\) \\
29 & 0.8234 & 0.8108 & \(+0.0126\) & 0.8232 & 0.8108 & \(+0.0124\) \\
41 & 0.8219 & 0.8122 & \(+0.0097\) & 0.8219 & 0.8122 & \(+0.0097\) \\
53 & 0.8239 & 0.8174 & \(+0.0065\) & 0.8130 & 0.8174 & \(-0.0044\) \\
67 & 0.8301 & 0.8180 & \(+0.0122\) & 0.8301 & 0.8180 & \(+0.0122\) \\
83 & 0.8255 & 0.8162 & \(+0.0094\) & 0.8255 & 0.8162 & \(+0.0094\) \\
\bottomrule
\end{tabular}
\end{table}

\begin{table}[!htbp]
\centering
\small
\caption{Per-structure Allen-15 Dice at the macro-selected best checkpoint, six seeds.
\(p_{\mathrm{exact}}\) is the unadjusted two-sided sign-flip value; \(p_{t,\mathrm{Holm}}\) is the
paired-\(t\) value after Holm correction across 15 structures. These are different tests. The
last two columns give the final-checkpoint difference and seed wins.}
\label{tab:app-structures}
\begin{tabular}{lcccccccc}
\toprule
& \multicolumn{6}{c}{Best checkpoint} & \multicolumn{2}{c}{Final checkpoint} \\
\cmidrule(lr){2-7}\cmidrule(lr){8-9}
Structure & \method{} & U-Net & \(\Delta\) & Wins & \(p_{\mathrm{exact}}\) & \(p_{t,\mathrm{Holm}}\) & \(\Delta\) & Wins \\
\midrule
AAVS1     & 0.8435 & 0.8036 & \(+0.0399\) & 6/6 & 0.03125 & 0.109 & \(+0.0578\) & 6/6 \\
ACTN1     & 0.7536 & 0.7518 & \(+0.0019\) & 4/6 & 0.46875 & 1.000 & \(-0.0036\) & 4/6 \\
ATP2A2    & 0.8117 & 0.7872 & \(+0.0246\) & 6/6 & 0.03125 & 0.023 & \(+0.0119\) & 4/6 \\
CETN2     & 0.5928 & 0.6295 & \(-0.0367\) & 1/6 & 0.0625 & 0.183 & \(-0.0381\) & 1/6 \\
DSP       & 0.6939 & 0.6771 & \(+0.0168\) & 6/6 & 0.03125 & 0.143 & \(+0.0031\) & 5/6 \\
FBL       & 0.9461 & 0.9418 & \(+0.0043\) & 3/6 & 0.375 & 1.000 & \(+0.0092\) & 4/6 \\
HIST1H2BJ & 0.8292 & 0.8253 & \(+0.0039\) & 3/6 & 0.53125 & 1.000 & \(+0.0095\) & 4/6 \\
LAMP1     & 0.8770 & 0.8563 & \(+0.0207\) & 6/6 & 0.03125 & 0.183 & \(+0.0254\) & 6/6 \\
LMNB1     & 0.9146 & 0.9027 & \(+0.0120\) & 6/6 & 0.03125 & 0.109 & \(+0.0197\) & 6/6 \\
MYH10     & 0.7705 & 0.7665 & \(+0.0040\) & 4/6 & 0.25 & 1.000 & \(+0.0029\) & 3/6 \\
NPM1      & 0.9266 & 0.9208 & \(+0.0059\) & 3/6 & 0.40625 & 1.000 & \(+0.0066\) & 3/6 \\
NUP153    & 0.8851 & 0.8889 & \(-0.0038\) & 3/6 & 0.6875 & 1.000 & \(-0.0028\) & 2/6 \\
SEC61B    & 0.8697 & 0.8376 & \(+0.0321\) & 6/6 & 0.03125 & 0.023 & \(+0.0340\) & 6/6 \\
TOMM20    & 0.8741 & 0.8648 & \(+0.0093\) & 4/6 & 0.28125 & 1.000 & \(+0.0068\) & 4/6 \\
TUBA1B    & 0.7816 & 0.7757 & \(+0.0059\) & 5/6 & 0.0625 & 0.183 & \(+0.0018\) & 5/6 \\
\bottomrule
\end{tabular}
\end{table}

\paragraph{UNETR control.}
The MONAI UNETR control maps the raw crop to the common token grid with a learned
\(1\times4\times4\) stem and then applies its standard \(16^3\) transformer patching, which leaves
64 transformer tokens; it does not match \method{}'s 262k attention tokens. Its 20k-step run
continues the 10k trajectory with the same validation cadence, so it receives more updates and
more checkpoint-selection opportunities. Its best Dice improves from 0.7430 to 0.7724, still
below the three-seed \method{} value of 0.8228 in all three seed pairs. The result is bounded by
this recipe: it is not a fully tuned UNETR, Swin UNETR, or nnFormer benchmark.

\subsection{Branch interventions}
\label{app:interventions}

Table~\ref{tab:app-interventions} reports the corrected post-training interventions for both
model families. If summary keys and values receive the same permutation \(P_S\), then
\(\softmax(SP_S^\top)P_S\widetilde{V}=\softmax(S)\widetilde{V}\), so the common permutation is a
no-op up to floating-point differences. Key-only and value-only permutations break the
key--value pairing. Same-layer cross-sample replacement removes sample-specific global
information but also shifts the summary distribution, so it is not a uniquely geometric
intervention. Every patch binds state separately for each attention layer. The two families
differ in how zeroing interacts with the mixer: Allen-6's convex gate leaves the surviving local
output scaled by its gate weight, whereas Allen-15 keeps the local output as a residual. Their
magnitudes are therefore not one architectural effect.

\begin{table}[!htbp]
\centering
\small
\caption{Post-training interventions on best checkpoints, three seeds each; mean \(\pm\) sample
s.d. and change from the unchanged model. Allen-15: all 120 validation crops; Allen-6: all
48.}
\label{tab:app-interventions}
\begin{tabular}{lcccc}
\toprule
& \multicolumn{2}{c}{Allen-15 (residual local)} & \multicolumn{2}{c}{Allen-6 (convex mixer)} \\
\cmidrule(lr){2-3}\cmidrule(lr){4-5}
Intervention & Dice & \(\Delta\) & Dice & \(\Delta\) \\
\midrule
Unchanged                         & 0.8228 \(\pm\) 0.0008 & ---          & 0.8531 \(\pm\) 0.0005 & --- \\
Common summary K/V permutation    & 0.8230 \(\pm\) 0.0006 & \(+0.0001\)  & 0.8532 \(\pm\) 0.0005 & \(0.0000\) \\
Key-only summary permutation      & 0.8227 \(\pm\) 0.0008 & \(-0.0002\)  & 0.8305 \(\pm\) 0.0038 & \(-0.0226\) \\
Value-only summary permutation    & 0.8228 \(\pm\) 0.0003 & \(0.0000\)  & 0.8197 \(\pm\) 0.0038 & \(-0.0334\) \\
Zero compressed output            & 0.8135 \(\pm\) 0.0034 & \(-0.0094\)  & 0.6775 \(\pm\) 0.0184 & \(-0.1756\) \\
Same-layer cross-sample summaries & 0.8071 \(\pm\) 0.0052 & \(-0.0157\)  & 0.6700 \(\pm\) 0.0100 & \(-0.1831\) \\
\bottomrule
\end{tabular}
\end{table}

\subsection{The Allen-6 study: capacity gap, structures, and collar controls}
\label{app:allen6}

The headline of the reviewed version came from Allen-6. There, \method{} reaches \(0.8531\pm0.0005\) best and
\(0.8429\pm0.0136\) final macro Dice over three seeds, against \(0.7553\pm0.0042\) and
\(0.7181\pm0.0062\) for a local--global model with a learned full-grid embedding of 158.5M
parameters (\method{}: 24.3M). The \(+0.0978\) best-Dice gap is not a measurement of coordinate
geometry. A separately trained axial local-only model reaches \(0.8387\) best and \(0.8379\) final,
so the compressed branch contributes \(+0.0144\) and \(+0.0050\), and the rest reflects a learned
full-grid embedding that is weak at this scale. Table~\ref{tab:app-allen6-structures} shows the
same pattern per structure. Against the learned embedding, \method{} improves all six structures
at both horizons. Against a U-Net, the picture is different: the U-Net has the higher final Dice
on five of six structures, and \method{} wins only on ATP2A2. The Allen-6 study supports a
capacity statement about the learned full-grid embedding, not a positional or architecture
claim.

\begin{table}[!htbp]
\centering
\small
\caption{Allen-6 per-structure Dice (best / final), from the reviewed version. W+G denotes the
local--global operator. Bold marks the higher final Dice between \method{} and U-Net.}
\label{tab:app-allen6-structures}
\begin{tabular}{lcccc}
\toprule
Structure & \method{} (axial W+G) & Learned PE W+G & Axial local-only & U-Net \\
\midrule
TOMM20  & 0.892 / 0.876 & 0.783 / 0.729 & 0.857 / 0.860 & 0.892 / \textbf{0.893} \\
LMNB1   & 0.921 / 0.893 & 0.796 / 0.769 & 0.909 / 0.906 & 0.903 / \textbf{0.907} \\
ST6GAL1 & 0.855 / 0.852 & 0.807 / 0.794 & 0.834 / 0.839 & 0.861 / \textbf{0.864} \\
GJA1    & 0.742 / 0.740 & 0.683 / 0.646 & 0.759 / 0.762 & 0.761 / \textbf{0.765} \\
LAMP1   & 0.889 / 0.881 & 0.802 / 0.765 & 0.865 / 0.866 & 0.883 / \textbf{0.886} \\
ATP2A2  & 0.819 / \textbf{0.815} & 0.661 / 0.605 & 0.808 / 0.794 & 0.801 / 0.778 \\
\bottomrule
\end{tabular}
\end{table}

A natural objection is that the compressed-branch gain comes from the external collar content
the summaries draw on rather than from global geometry. Models trained with the collar zeroed or
shuffled stay close to the real-collar model (Table~\ref{tab:app-allen6-collar}), so collar
content is not a clean causal driver. The small best-final gaps of the local--global rows, next
to the larger gap of the learned-PE model, also indicate that the axial checkpoints are not
simply less stable.

\begin{table}[!htbp]
\centering
\small
\caption{Allen-6 collar and stability controls from the reviewed version (trained models, three
seeds).}
\label{tab:app-allen6-collar}
\begin{tabular}{lccc}
\toprule
Model & Best Dice & Final Dice & Best \(-\) final \\
\midrule
\method{}, real collar           & 0.8531 & 0.8429 & 0.0103 \\
\method{}, zeroed collar         & 0.8447 & 0.8422 & 0.0025 \\
\method{}, shuffled collar       & 0.8503 & 0.8417 & 0.0086 \\
Axial local-only, real collar    & 0.8387 & 0.8379 & 0.0008 \\
Learned PE W+G, real collar      & 0.7553 & 0.7181 & 0.0372 \\
\bottomrule
\end{tabular}
\end{table}

\subsection{Fluorescence controls}
\label{app:fluorescence}

Table~\ref{tab:app-threshold-baselines} gives the held-out threshold controls. In a related
six-structure, depth-4 control, 50\% of the supervised-region target-fluorescence voxels are
dropped at random; this is not removal of the whole target channel. \method{} reaches
\(0.8355\pm0.0032\) best Dice and a TokenGrid U-Net \(0.8415\pm0.0020\). With brightfield as the
only input, best Dice is \(0.1794\) and \(0.1866\), respectively. The results support
fluorescence-assisted modeling, not label-free segmentation, and neither a threshold gap nor a
channel audit establishes independence from the automated label pipeline.

\subsection{Auxiliary CT classification}
\label{app:nodule}

NoduleMNIST3D \citep{yangMedMNIST2023} is a lower-token diversity check, not a high-token
segmentation result. It has 1{,}158 training and 165 validation \(28^3\) volumes; the 310 test
volumes are not used. A one-layer, width-128 backbone with four query heads, two KV heads, head
dimension 32, unit patches, and \(2\times4\times4\) native blocks trains for 2{,}000 updates at
batch size eight, validating every 250. Balanced accuracy is the unweighted mean of per-class
recall. Across six seeds, axial RoPE exceeds packed RoPE by 0.1433 best balanced accuracy
(\(0.7422\pm0.0174\) against \(0.5988\pm0.0367\)), but separable PE exceeds axial RoPE by 0.0456
(\(0.7878\pm0.0196\)); both contrasts have exact \(p=0.031\). The stronger coordinate-indexed
alternative is reported alongside the packed comparison, and the result supports neither
universal axial superiority nor an extrapolation to high-token organ segmentation.

\subsection{What the mechanism probes establish}
\label{app:probes}

The placement probe (Figure~\ref{fig:app-probes}a) reaches mean accuracy 0.9997 for token-before rotation, against 0.5000 for
both controls in every seed. Across the checked pairs the largest control summary-key
discrepancy is \(2.4\times10^{-7}\), and so is the largest token-before summary-value
discrepancy; these are finite-precision collision checks confirming that the controls receive
identical summaries and that the token-before arm receives identical values. In the long-range
probe the radius-one local--global model reaches 1.0000, local-only models of radius zero to two
stay at 0.5000, and a radius-three local model that reaches the evidence also solves the task,
which shows that the rule is learnable without a global branch once the evidence is local.
Zeroing the remote bit returns the full model to 0.5000, and flipping it gives 0.0000 accuracy
against the original labels. Each full-versus-inaccessible-local contrast has exact
\(p=0.031\). Neither construction establishes that natural segmentation requires the isolated
channel.

\FloatBarrier
\subsection{Exploratory single-seed gates from the reviewed version}
\label{app:exploratory}

The reviewed version compared a wider set of positional mechanisms in single-seed short-schedule Allen
gates. Table~\ref{tab:app-exploratory} retains the rows that remain valid, for completeness.
Its packed-sequence RoPE rows are withdrawn because that arm had no active rotary pairs
(Appendix~\ref{app:corrections}). The 262k/depth-8 pair differs in capacity and is not a
matched control. These gates motivated the matched three-seed controls of
Table~\ref{tab:fixed-layout}, which supersede them; in particular, they should not be read as
evidence that axial RoPE beats compact learned coordinates on a fixed layout.

\begin{table}[!htbp]
\centering
\small
\caption{Exploratory single-seed short gates from the reviewed version (best / final Dice). The 65k
rows include training time and peak memory; the dense relative-bias row materializes its bias.
Superseded by the matched controls of Table~\ref{tab:fixed-layout}.}
\label{tab:app-exploratory}
\begin{tabular}{llccl}
\toprule
Setting & Mechanism & Best / final & Params & Notes \\
\midrule
262k, depth 1 & Learned full-grid PE       & 0.7086 / 0.6525 & 137.3M & \\
262k, depth 1 & Learned separable 3D PE    & 0.7151 / 0.7151 & 3.18M & \\
262k, depth 1 & Learned low-rank 3D PE     & 0.6891 / 0.6852 & 3.07M & \\
262k, depth 1 & Coordinate Fourier MLP     & 0.6942 / 0.6942 & 3.13M & \\
262k, depth 1 & Axial RoPE                 & 0.7375 / 0.7041 & 3.06M & \\
\midrule
262k, depth 1 & Token-before-compression   & 0.7378 / 0.7048 & --- & placement gate \\
262k, depth 1 & Summary-center RoPE        & 0.7384 / 0.6947 & --- & placement gate \\
262k, depth 1 & Unrotated summary keys     & 0.7333 / 0.6830 & --- & placement gate \\
\midrule
262k, depth 8 & Learned full-grid PE       & 0.7185 / 0.7185 & 158.5M & 2k steps, unmatched \\
262k, depth 8 & Axial RoPE                 & 0.7960 / 0.7960 & 24.3M & 2k steps, unmatched \\
\midrule
65k, depth 1 & Learned full-grid PE        & 0.6939 / 0.6939 & 36.6M & 71.8\,s, 2.60\,GiB \\
65k, depth 1 & 3D relative position bias   & 0.7055 / 0.6420 & 3.06M & 588.0\,s, 3.35\,GiB \\
65k, depth 1 & 3D ALiBi-style bias         & 0.7098 / 0.6691 & 3.06M & 74.3\,s, 2.97\,GiB \\
65k, depth 1 & Axial RoPE                  & 0.7346 / 0.7154 & 3.06M & 79.4\,s, 2.27\,GiB \\
\bottomrule
\end{tabular}
\end{table}

\FloatBarrier
\subsection{Corrections to the reviewed version}
\label{app:corrections}

Table~\ref{tab:app-corrections} lists every material change from the reviewed version. Each narrows or replaces one of its claims; none restores a fixed-layout positional or placement
superiority claim. The withdrawn results are not relabeled as valid evidence elsewhere in this
paper.

\begin{table}[!htbp]
\centering
\small
\caption{Corrections to the experimental record and descriptions of the reviewed version.}
\label{tab:app-corrections}
\begin{tabularx}{\linewidth}{L{2.6cm}YY}
\toprule
Item & Reviewed version & This version \\
\midrule
Packed-sequence RoPE controls & Rows at 262k and 65k (and in MSD transfer) reported packed RoPE behaving like no positional encoding. & That arm allocated zero rotary pairs, so it tested no sequence-index RoPE. All such rows are withdrawn and replaced by a runtime-attested 32-pair control (Table~\ref{tab:fixed-layout}), which matches axial RoPE on a fixed layout. \\
Summary-shuffle ablation & Shuffling compressed summaries dropped Allen-6 Dice to 0.6858. & The evaluator's per-layer patch closure was late-bound and fed each layer another layer's summaries. Withdrawn; a correctly bound common K/V permutation is a no-op (Table~\ref{tab:app-interventions}). \\
Threshold baseline & Tuned threshold, 0.5582 Dice. & Fitted on the validation crops it was scored on. Replaced by a training-fitted threshold scored on held-out crops (0.5576). \\
Headline comparison & \(+0.0978\) Dice over learned-PE local--global attention at 262k/depth 8. & A capacity result (158.5M vs.\ 24.3M parameters; Appendix~\ref{app:allen6}). Matched controls and repacking carry the positional argument. \\
Layout and summaries & Two-level block-Morton layout; 16{,}384 summaries per query; uniform or fixed block weights. & Reported Allen rows use \texttt{raster\_block} (\(1\times1\times16\) strips) and a second pooling stage, so queries see 1{,}024 summaries; weights are learned and content-dependent with slot biases. The additive-bias element count is corrected accordingly. \\
Packing theorem & Equivariance under coordinate-defined neighborhoods and blocks. & Also requires preserved fine and coarse memberships and transported slot identities (Theorem~\ref{thm:packing}); named-layout switches fall outside it. \\
Branch mixer & Convex softmax mixture throughout. & Allen-15 and MSD use a residual local branch with a gated global branch; the convex mixer applies to Allen-6 only. \\
Allen-15 supervision & Dice on a fixed central target region. & The central region applies to Allen-6; Allen-15 supervises and scores the full token grid. \\
Execution & A fused forward/backward block. & Custom Triton local kernels, autograd summary construction, and SDPA global attention; no whole-operator fused backward. \\
Timing table & Pooling factor 64; \(6.49\times\) and \(125.45\times\) faster than dense at 65k and 1M. & Re-measured at the trained configuration's pooling factor 16 (Table~\ref{tab:systems}): \(6.43\times\) and \(99.56\times\), and 7--31\% slower than window attention. \\
External transfer & Single-seed MSD gates at 65k tokens. & Replaced by 18 three-seed runs at 262k tokens with repacking and named-layout interventions (Table~\ref{tab:medical}). \\
\bottomrule
\end{tabularx}
\end{table}

\clearpage
\section{Optional Selected Refinement and Systems Extensions}
\label{app:selected-systems}

The main method is local plus compressed global attention with axial grid-coordinate RoPE. The
codebase also supports selected sparse refinement, query-tiled execution, and multi-atom
summaries; these are extensions, deliberately separated here from the main evidence chain.

\subsection{Execution behind the reported runs}
\label{app:systems-execution}

The natural-data runs use custom Triton local attention, PyTorch autograd for fine and coarse
pooling, and SDPA for query-to-summary attention. SDPA kernel selection depends on data type,
shape, hardware, and software, so we do not assert the absence of a global score allocation for
every backend. The local kernel avoids an all-query gathered \(N\times w\times d_h\) feature tensor
and reconstructs probabilities in the backward pass. Table~\ref{tab:systems} times one full
attention layer, forward and backward, including its projections and branch processing; it does
not include an eight-layer network, the decoder, the loss, or an optimizer step. Dense SDPA and
window attention are execution comparators, not equal-accuracy segmentation models. The timing
table of the reviewed version used pooling factor 64, a smaller summary budget than the trained
models, and is not substituted for Table~\ref{tab:systems}.

\subsection{Selected sparse refinement}
\label{app:selected-refinement}

For a query block \(Q_b\), the optional selected branch pools queries,
\(\bar{q}_b=\operatorname{Pool}_{i\in Q_b}(q_i^R)\), scores the compressed summaries,
\(p_{b,j}=\softmax_j(\bar{q}_b^\top\widetilde{k}_j^R/\sqrt{d_h})\), selects the top-\(k\) blocks
\(I_b=\operatorname{TopK}_j(p_{b,j},k_{\mathrm{sel}})\), and lets each query in \(Q_b\) attend to the
raw keys and values of those blocks,
\[
    o_i^{\mathrm{sel}} = \attn\big(q_i^R,\,K^{\mathrm{sel},R}_{\cup_{j\in I_b}B_j},\,
    V^{\mathrm{sel}}_{\cup_{j\in I_b}B_j}\big).
\]
For grouped-query attention, selected block indices are shared within each KV group, so selected
K/V access remains blockwise and hardware-aligned. The branch is disabled in every reported
segmentation run; it is retained for systems reach.

\subsection{High-token systems reach}
\label{app:systems-reach}

Query-tiled execution with selected refinement runs at 4M and 8M tokens on one H200
(Table~\ref{tab:app-high-token-reach}). These rows, from the reviewed version, show execution reach,
not practical throughput: they are slow, skip some diagnostics above the configured
auxiliary-materialization threshold, and do not use the configuration of the trained models.

\begin{table}[!htbp]
\centering
\small
\caption{Representative high-token reach rows from the reviewed version (selected refinement,
query-tiled). Not practical-throughput claims.}
\label{tab:app-high-token-reach}
\begin{tabular}{rlrrr}
\toprule
Tokens & Route granularity & Fwd+bwd (ms) & Peak GiB & Tokens/s \\
\midrule
4{,}194{,}304 & query-block shared & 45{,}113 & 51.50 & 92{,}973 \\
4{,}194{,}304 & token-exact        & 113{,}675 & 51.50 & 36{,}897 \\
8{,}388{,}608 & query-block shared & 176{,}190 & 102.04 & 47{,}611 \\
8{,}388{,}608 & token-exact        & 452{,}953 & 102.04 & 18{,}520 \\
\bottomrule
\end{tabular}
\end{table}

\subsection{Optional explicit-streaming backend}
\label{app:streamed}

The code release adds an opt-in Triton backend with explicit backward through learned summary
construction and both attention branches. Fine-summary construction and coarse pooling are
fused; global attention streams summary tiles with online softmax; the backward pass recomputes
pooling weights and attention probabilities; and saved state contains feature-sized and
summary-sized tensors, indices, and row normalizers, not an \(N\times M\) score tensor or an
\(N\times w\times d_h\) gathered tensor. Execution is staged across several kernels rather than
fused into one launch. On an H200 the backend passes 30 correctness tests covering outputs and
input gradients across both named layouts, three precisions, ragged and empty blocks, and
repacking with transported metadata, the same tests under CUDA memcheck, and 19 CPU regression
tests; the largest observed relative \(L_2\) errors are \(6.0\times10^{-7}\) (fp32),
\(7.0\times10^{-4}\) (fp16), and \(5.1\times10^{-3}\) (bf16). The backend was not used for any
training run, checkpoint evaluation, or timing reported in this paper, and we make no speed
claim for it.

\subsection{Environment, compute, and artifacts}
\label{app:environment}

Experiments ran with Python 3.12.3, PyTorch 2.11.0 (CUDA 12.6), Triton 3.6.0, MONAI 1.5.2, NumPy
2.4.4, and SciPy 1.17.1 on NVIDIA H200 GPUs with 143{,}771\,MiB each, one GPU per training run.
Logged single-GPU wall times, including evaluation, were 1.84--4.29 hours per run for the six
Allen-15 \method{} runs of the U-Net comparison (17.68 GPU-hours in total), 0.23--0.73 hours for
the six U-Net runs (2.90), and 0.27--0.31 hours for the 18 medical runs (5.33). These partial
totals exclude other positional runs, inference audits, preprocessing, abandoned runs, and
kernel development; a complete aggregate budget for the project is not available. Shared-device
interference can affect logged times. The evidence package records configurations, run and
checkpoint identifiers, analysis outputs, and source digests; tables and figures are generated
from completed result records, and checkpoints are selected only on the original layout before
any layout evaluator is applied.

\paragraph{Development assistance.}
LLM-based coding assistants supported implementation, auditing, the optional backend, and
manuscript revision. They are not components of the operator, labels, or prediction pipeline;
numerical results come from executed experiments and analyses, and generated assertions were
not accepted as empirical or mathematical evidence.

\clearpage
\section*{NeurIPS Paper Checklist}

\begin{enumerate}

\item {\bf Claims}
    \item[] Question: Do the main claims made in the abstract and introduction accurately reflect the paper's contributions and scope?
    \item[] Answer: \answerYes{}
    \item[] Justification: The abstract and Section~\ref{sec:intro} state the operator, its exact identities, and the empirical claim that coordinate indexing removes storage-order dependence; they state explicitly that no fixed-layout accuracy advantage, segmentation state of the art, or general speed law is claimed (\S\ref{subsec:claim-boundaries}).

\item {\bf Limitations}
    \item[] Question: Does the paper discuss the limitations of the work performed by the authors?
    \item[] Answer: \answerYes{}
    \item[] Justification: Section~\ref{sec:conclusion} has a limitations paragraph, \S\ref{subsec:claim-boundaries} states the claim boundary, and Appendix~\ref{app:claim-status} discusses validation selection, small medical populations, label-guided crops, fluorescence assistance, architecture-control limits, and reduced-quadratic global work.

\item {\bf Theory assumptions and proofs}
    \item[] Question: For each theoretical result, does the paper provide the full set of assumptions and a complete (and correct) proof?
    \item[] Answer: \answerYes{}
    \item[] Justification: Section~\ref{sec:method} states each result with its assumptions, and Appendix~\ref{app:theory} gives complete proofs, including the transport of both pooling stages and of canonical slots, the exact-arithmetic scope, and the separation's limiting construction.

\item {\bf Experimental result reproducibility}
    \item[] Question: Does the paper fully disclose all the information needed to reproduce the main experimental results of the paper to the extent that it affects the main claims and/or conclusions of the paper (regardless of whether the code and data are provided or not)?
    \item[] Answer: \answerYes{}
    \item[] Justification: Appendix~\ref{app:claim-status} specifies the data, preprocessing, crop and case counts, supervision, architecture, optimization, seeds, positional controls, layout interventions with their permutation seeds, checkpoint selection, and statistics; Appendix~\ref{app:additional-results} reports every seed.

\item {\bf Open access to data and code}
    \item[] Question: Does the paper provide open access to the data and code, with sufficient instructions to faithfully reproduce the main experimental results, as described in supplemental material?
    \item[] Answer: \answerNo{}
    \item[] Justification: All datasets are publicly available from their providers (Appendix~\ref{app:assets}). The code and evidence package are being prepared for public release; until a public repository is available, we do not claim open access.

\item {\bf Experimental setting/details}
    \item[] Question: Does the paper specify all the training and test details (e.g., data splits, hyperparameters, how they were chosen, type of optimizer, etc.) necessary to understand the results?
    \item[] Answer: \answerYes{}
    \item[] Justification: Section~\ref{subsec:results-contract} and Appendix~\ref{app:claim-status} distinguish the Allen-15, Allen-6, and medical contracts, the shared training budget, original-layout selection, and inference-only interventions; all reported scores are validation scores.

\item {\bf Experiment statistical significance}
    \item[] Question: Does the paper report error bars suitably and correctly defined or other appropriate information about the statistical significance of the experiments?
    \item[] Answer: \answerYes{}
    \item[] Justification: Tables report sample standard deviations across training seeds, and paired comparisons report exact sign-flip tests with their attainable resolution, paired-\(t\) intervals, and Holm-corrected per-structure tests (Appendix~\ref{app:statistics}). Seed variability is not presented as patient-level uncertainty or as evidence of equivalence.

\item {\bf Experiments compute resources}
    \item[] Question: For each experiment, does the paper provide sufficient information on the computer resources (type of compute workers, memory, time of execution) needed to reproduce the experiments?
    \item[] Answer: \answerYes{}
    \item[] Justification: Appendix~\ref{app:environment} gives the hardware, software versions, one-GPU-per-run setup, and logged wall times for the principal run families, and Table~\ref{tab:systems} gives layer latency and memory. It also states that the full project used additional, unaccounted compute.

\item {\bf Code of ethics}
    \item[] Question: Does the research conducted in the paper conform, in every respect, with the NeurIPS Code of Ethics \url{https://neurips.cc/public/EthicsGuidelines}?
    \item[] Answer: \answerYes{}
    \item[] Justification: The work uses existing public research datasets under their terms, discloses its limitations and corrections, and involves no new participants or clinical deployment.

\item {\bf Broader impacts}
    \item[] Question: Does the paper discuss both potential positive societal impacts and negative societal impacts of the work performed?
    \item[] Answer: \answerYes{}
    \item[] Justification: The work targets more reliable analysis of volumetric scientific and medical images; Section~\ref{sec:conclusion} and Appendix~\ref{app:assets} discuss automated-label bias, label-guided crops, small validation populations, and the risk of unvalidated medical or scientific use.

\item {\bf Safeguards}
    \item[] Question: Does the paper describe safeguards that have been put in place for responsible release of data or models that have a high risk for misuse (e.g., pretrained language models, image generators, or scraped datasets)?
    \item[] Answer: \answerNA{}
    \item[] Justification: The paper releases no high-risk generative model or scraped dataset; it directs users to the original data providers and states that the models are not validated clinical systems.

\item {\bf Licenses for existing assets}
    \item[] Question: Are the creators or original owners of assets (e.g., code, data, models), used in the paper, properly credited and are the license and terms of use explicitly mentioned and properly respected?
    \item[] Answer: \answerYes{}
    \item[] Justification: Appendix~\ref{app:assets} credits the datasets and principal software and states the Allen data terms, CC-BY-SA 4.0 for the Decathlon tasks, CC BY 4.0 for NoduleMNIST3D, and the software licenses; no scans are redistributed.

\item {\bf New assets}
    \item[] Question: Are new assets introduced in the paper well documented and is the documentation provided alongside the assets?
    \item[] Answer: \answerYes{}
    \item[] Justification: The code and evidence package to be released with the paper document configurations, run identifiers, analysis outputs, source digests, and, for the optional backend, its supported configurations, tests, and limitations (Appendices~\ref{app:streamed} and~\ref{app:environment}).

\item {\bf Crowdsourcing and research with human subjects}
    \item[] Question: For crowdsourcing experiments and research with human subjects, does the paper include the full text of instructions given to participants and screenshots, if applicable, as well as details about compensation (if any)?
    \item[] Answer: \answerNA{}
    \item[] Justification: No crowd workers or human participants are involved; the medical experiments use existing public dataset splits.

\item {\bf Institutional review board (IRB) approvals or equivalent for research with human subjects}
    \item[] Question: Does the paper describe potential risks incurred by study participants, whether such risks were disclosed to the subjects, and whether Institutional Review Board (IRB) approvals (or an equivalent approval/review based on the requirements of your country or institution) were obtained?
    \item[] Answer: \answerNA{}
    \item[] Justification: The work collects no new patient data and involves no participant intervention; it analyzes existing, publicly distributed research datasets.

\item {\bf Declaration of LLM usage}
    \item[] Question: Does the paper describe the usage of LLMs if it is an important, original, or non-standard component of the core methods in this research? Note that if the LLM is used only for writing, editing, or formatting purposes and does not impact the core methodology, scientific rigorousness, or originality of the research, declaration is not required.
    \item[] Answer: \answerNA{}
    \item[] Justification: LLMs are not a component of the method, labels, or prediction pipeline. Coding and editorial assistance is nonetheless disclosed in Appendix~\ref{app:environment}.

\end{enumerate}

\end{document}